\documentclass[letterpaper]{article} 
\usepackage{aaai24}  
\usepackage{times}  
\usepackage{helvet}  
\usepackage{courier}  
\usepackage[hyphens]{url}  
\usepackage{graphicx} 
\usepackage{natbib}  
\usepackage{caption} 
\usepackage{algorithm}
\usepackage{algorithmic}

\usepackage{newfloat}
\usepackage{listings}
\DeclareCaptionStyle{ruled}{labelfont=normalfont,labelsep=colon,strut=off} 
\floatstyle{ruled}
\newfloat{listing}{tb}{lst}{}
\floatname{listing}{Listing}

\usepackage{amsthm}
\usepackage{amsmath}
\usepackage[switch]{lineno}
\usepackage{amsfonts}   
\usepackage{multirow}
\usepackage{caption}
\usepackage{subcaption}
\usepackage{xcolor}

\newtheorem{theorem}{Theorem}
\newtheorem{lemma_apd}{Lemma}

\newcommand{\sterling}{{\textsc{Sterling}}}

\title{\sterling: Synergistic Representation Learning on Bipartite Graphs}

\author{
Baoyu Jing\textsuperscript{\rm 1},
Yuchen Yan\textsuperscript{\rm 1},
Kaize Ding\textsuperscript{\rm 2},
Chanyoung Park\textsuperscript{\rm 3},\\
Yada Zhu\textsuperscript{\rm 4},
Huan Liu\textsuperscript{\rm 5},
Hanghang Tong\textsuperscript{\rm 1}
}
\affiliations {
\textsuperscript{\rm 1}University of Illinois at Urbana-Champaign\\
\textsuperscript{\rm 2}Northwestern University\\
\textsuperscript{\rm 3}Korea Advanced Institute of Science \& Technology\\
\textsuperscript{\rm 4}MIT-IBM Watson AI Lab, IBM Research\\
\textsuperscript{\rm 5}Arizona State University\\
baoyuj2@illinois.edu, 
yucheny5@illinois.edu,
kaize.ding@northwestern.edu,
cy.park@kaist.ac.kr,\\
yzhu@us.ibm.com,
huanliu@asu.edu,
htong@illinois.edu
}

\begin{document}
\maketitle

\begin{abstract}
A fundamental challenge of bipartite graph representation learning is how to extract informative node embeddings.
Self-Supervised Learning (SSL) is a promising paradigm to address this challenge. 
Most recent bipartite graph SSL methods are based on contrastive learning which learns embeddings by discriminating positive and negative node pairs.
Contrastive learning usually requires a large number of negative node pairs, which could lead to computational burden and semantic errors. 
In this paper, we introduce a novel synergistic representation learning model (\sterling) to learn node embeddings without negative node pairs. 
\sterling\ preserves the unique local and global synergies in bipartite graphs.
The local synergies are captured by maximizing the similarity of the inter-type and intra-type positive node pairs, and the global synergies are captured by maximizing the mutual information of co-clusters.
Theoretical analysis demonstrates that \sterling\ could improve the connectivity between different node types in the embedding space.
Extensive empirical evaluation on various benchmark datasets and tasks demonstrates the effectiveness of \sterling\ for extracting node embeddings.
\end{abstract}

\section{Introduction}
The bipartite graph is a powerful representation formalism to model interactions between two types of nodes, which has been used in various real-world applications.
In recommender systems \cite{wang2021graph,wei2022comprehensive,zhou2021high},  users, items and their interactions (e.g., buy) naturally formulate a bipartite graph;
in drug discovery \cite{pavlopoulos2018bipartite}, chemical interactions between drugs and proteins also formulate a bipartite graph; 
in information retrieval \cite{he2016birank}, clickthrough between queries and web pages can be modeled by a bipartite graph.

A fundamental challenge for bipartite graphs is how to extract informative node embeddings that can be easily used for downstream tasks (e.g., link prediction).
In recent years, Self-Supervised Learning (SSL) has become a prevailing paradigm to learn embeddings without human-annotated labels \cite{wu2021self,DBLP:conf/kdd/ZhengXZH22}.  
Despite its great performance on downstream tasks (e.g., node classification), most of the methods are designed for homogeneous graphs \cite{you2020graph,zhu2021graph,feng2022adversarial,DBLP:conf/cikm/ZhouZF0H22,zheng2021deeper,wang2023characterizing,ding2023eliciting} and heterogeneous graphs \cite{park2020unsupervised,jing2021hdmi,wang2021self,fu2020view,yan2022dissecting}.
These methods are usually sub-optimal to bipartite graphs \cite{gao2018bine}, and therefore, several methods have been specifically proposed for bipartite graphs.
BiNE \cite{gao2018bine} and BiANE \cite{huang2020biane} learn embeddings by maximizing the similarity of neighbors sampled by random walks; 
NeuMF \cite{he2017neural} and NGCF \cite{wang2019neural} train neural networks by reconstructing the edges; 
BiGI \cite{cao2021bipartite}, SimGCL \cite{yu2022graph} and COIN \cite{jing2022coin} further improve the quality of embeddings via contrastive learning.

\begin{figure}[t!]
    \centering
    \includegraphics[width=.39\textwidth]{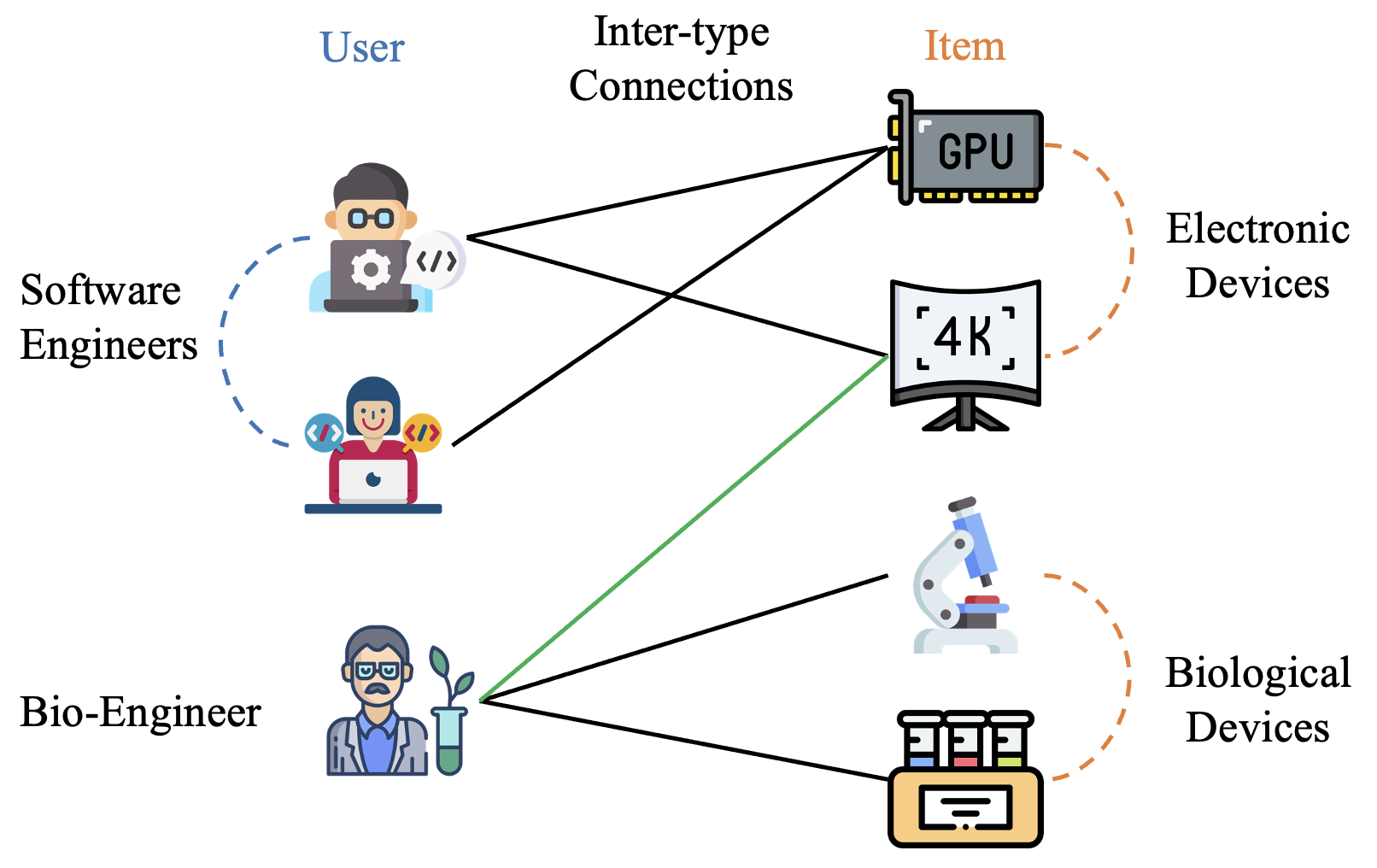}
    \caption{Example of a bipartite graph and its unique local and global properties. Locally, the dashed curves are 
    \emph{implicit} intra-type connections. Globally, the green line shows the inter-connection between the co-clusters, i.e., bio-engineer and electronic devices.}
    \label{fig:example}
\end{figure}

The aforementioned methods are mainly based on contrastive learning, which learns node embeddings by discriminating positive node pairs (e.g., local neighbors) and negative node pairs (e.g., unconnected nodes).
The success of contrastive learning heavily relies on the careful treatment of large-scale negative node pairs 
(e.g., adaptive selection of negative samples) \cite{grill2020bootstrap}.
There still lacks a principled mechanism to efficaciously construct desirable negative node pairs, which may cause computational burden \cite{thakoor2021large,DBLP:conf/sdm/ZhengZH23} and semantic errors \cite{li2022graph}.
Recently, BGRL \cite{thakoor2021large} and AFGRL \cite{lee2022augmentation} propose to learn node embeddings without negative pairs for homogeneous graphs via bootstrapping.
However, these non-contrastive methods cannot be effectively applied to bipartite graphs due to their incapability of capturing unique properties of bipartite graphs.

As shown in Fig. \ref{fig:example}, bipartite graphs have their unique local and global properties. 
Locally, besides the explicit inter-type connections (solid lines), the implicit intra-type synergies are also important (the dashed curves).
In Fig. \ref{fig:example}, it is very likely that a user (the girl software engineer) will buy the item (the monitor) that was bought by a similar user (the boy software engineer).
Globally, the two node types are inter-connected, and thus their cluster-level semantics are inherently synergistic. 
For example, users and items in Fig. \ref{fig:example} can be clustered based on their professions (software engineers and bio-engineer) and usage (electronic devices and biological devices). 
The relationship among these co-clusters is not a simple one-to-one correspondence, e.g., software engineers $\sim$ electronic devices and bio-engineer $\sim$ biological devices.
In fact, they are inherently inter-connected: the bio-engineer cluster is also connected to the electronic device cluster.
Neither independently nor equally treating these co-clusters could capture such a synergy.

In this paper, we introduce a novel synergistic representation learning model (\sterling) for bipartite graphs.
Compared with bipartite contrastive learning methods, \sterling\ is a non-contrastive SSL method that does not require negative node pairs. 
Compared with SSL methods on general graphs, \sterling\ captures both local and global properties of bipartite graphs.
For the local synergies, we maximize the similarity of the positive node embedding pairs.
When creating positive node pairs, not only do we consider the inter-type synergies (e.g., connected users and items), but also the intra-type synergies (e.g., similar users).
For the global synergies, we introduce a simple end-to-end deep co-clustering model to produce co-clusters of the two node types, 
and we capture the global cluster synergies by maximizing the mutual information of the co-clusters.
We further present the theoretical analysis and empirical evaluation of \sterling.
In theoretical analysis, we prove that maximizing the mutual information of co-clusters increases the mutual information of the two node types in the embedding space.
This theorem indicates that maximizing the mutual information of co-clusters could improve the connectivity of two node types in the embedding space.
In empirical evaluation, we extensively evaluate \sterling\ on various real-world datasets and tasks to demonstrate its effectiveness.

The major contributions are summarized as follows:
\begin{itemize}
    \item A novel SSL model (\sterling) is proposed for bipartite graphs, which preserves local and global synergies of bipartite graphs and does not require negative node pairs.
    \item Theoretical analysis shows that \sterling\ improves the connectivity of the two node types in embedding space.
    \item Extensive evaluation is conducted to demonstrate the effectiveness of the proposed \sterling. 
\end{itemize}

\begin{figure*}[t!]
    \centering
    \includegraphics[width=.7\textwidth]{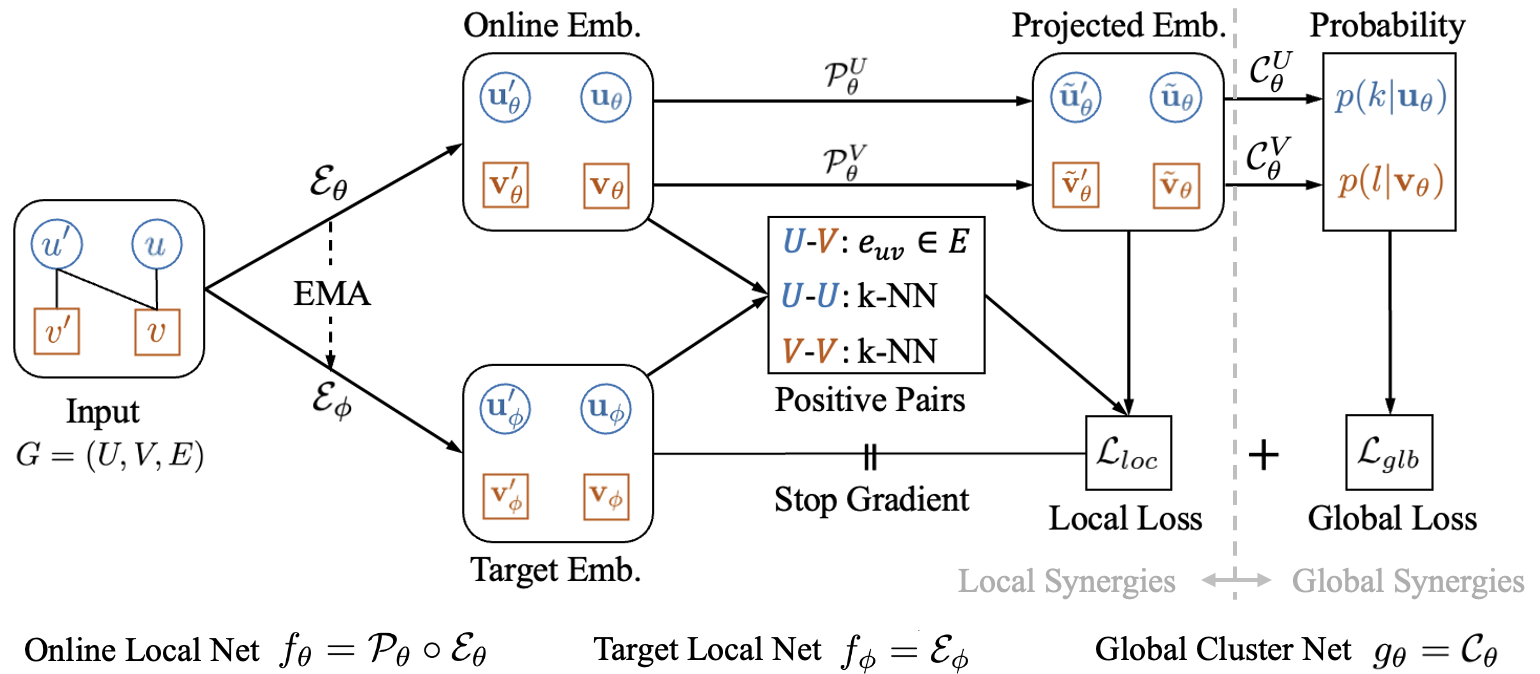}
    \caption{Overview of \sterling. $\mathcal{E}$, $\mathcal{P}$ and $\mathcal{C}$ are the encoder, projector and cluster network. $\theta$ and $\phi$ are parameters of the online and target networks. $\theta$ is updated by optimizing objectives, $\phi$ is updated via Exponential Moving Average (EMA) of $\theta$. For details, please refer to the methodology section.}
    \label{fig:overview}
\end{figure*}

\section{Related Work}
\textbf{Graph Embedding.}
Graphs are ubiquitous in real-world applications, e.g., social network \cite{goyal2018graph,zheng2023fairgen,yan2021bright,zeng2023hierarchical}, finance \cite{DBLP:conf/kdd/ZhouZ0H20,jing2021network} and natural language processing \cite{wu2021deep,jing2021multiplex,yan2021dynamic}.
A fundamental challenge of graph representation learning is to extract informative node embeddings \cite{wu2021self,zeng2023generative}.
Early methods DeepWalk \cite{perozzi2014deepwalk}, node2vec \cite{grover2016node2vec}, LINE \cite{tang2015line} use the random walk to sample node pairs and maximizes their similarities.
Contrastive learning methods, e.g., DGI \cite{velickovic2019deep} GraphCL \cite{you2020graph}, GCA \cite{zhu2021graph} and ARIEL \cite{feng2022ariel}, learn embeddings by discriminating positive and negative node pairs. 
Recently, some non-contrastive methods are proposed.
BGRL \cite{thakoor2021large} and AFGRL \cite{lee2022augmentation} learn embeddings via bootstrapping, and GraphMAE \cite{hou2022graphmae} learns embeddings by reconstructing the full graph from the masked graph.
All of these methods are designed for homogeneous graphs, yet many real-world graphs are heterogeneous \cite{hu2020heterogeneous,yan2023trainable,zeng2023parrot}.
Metapath2vec \cite{dong2017metapath2vec} extends node2vec to heterogeneous graphs.
DMGI \cite{park2020unsupervised} and HDMI \cite{jing2021hdmi} extends DGI to heterogeneous graphs.
HeCo \cite{wang2021self} and X-GOAL \cite{jing2022x} introduces co-contrastive learning and prototypical contrastive learning for heterogeneous graphs.
Although they achieved impressive performance on downstream tasks (e.g., classification), they are not tailored for bipartite graphs and usually have sub-optimal performance compared with bipartite graph methods on tasks such as recommendation and link prediction. 

\noindent
\textbf{Bipartite Graph Embedding.}
Bipartite graphs have been widely used to model interactions between two disjoint node sets \cite{he2020lightgcn,mao2021ultragcn,yan2023reconciling,zhang2023contrastive}.
Early methods such as BiNE \cite{gao2018bine} and BiANE \cite{huang2020biane} learn node embeddings based on biased random walks. 
IGE \cite{zhang2017learning} learns node embeddings based on the direct connection between nodes and edge attributes.
PinSage \cite{ying2018graph} combines graph convolutional network with random walks.
Collaborative filtering methods are also popular for bipartite graphs \cite{wei2020fast}.
NeuMF \cite{he2017neural} is the first neural collaborative filtering method.
NGCF \cite{wang2019neural} incorporates high-order collaborative signals to improve the quality of embeddings.
DirectAU \cite{wang2022towards} learns node embeddings from the perspective of alignment and uniformity.
In recent years, contrastive learning has been applied to bipartite graph.
BiGI \cite{cao2021bipartite} extends DGI to from homogeneous graphs to bipartite graphs.
SimGCL \cite{yu2022graph} and COIN \cite{jing2022coin} use InfoNCE \cite{oord2018representation} as the loss.
AdaGCL \cite{jiang2023adaptive} uses contrastive loss as an auxiliary signal for the recommendation task.
Different from these contrastive methods, \sterling\ does not require any negative node pairs, and further explores the synergies among co-clusters.

\noindent
\textbf{Co-Clustering.}
Co-clustering aims to partition rows and columns of a co-occurrence matrix into co-clusters simultaneously.
In practice, co-clustering algorithms usually have impressive improvements over traditional one-way clustering algorithms \cite{xu2019deep}.
CCInfo \cite{dhillon2003information} co-clusters matrices via a mutual information based objective.
SCC \cite{dhillon2001co} is based on spectral analysis.
BCC \cite{shan2008bayesian}, LDCC \cite{shafiei2006latent} and MPCCE \cite{wang2011nonparametric} are Bayesian approaches.
CCMod \cite{ailem2015co} obtains co-clusters by maximizing graph modularity.
SCMK \cite{kang2017twin} is a kernel based method.
DeepCC \cite{xu2019deep} is a deep learning based co-clustering method, which uses an auto-encoder to extract embeddings, a deep Gaussian mixture model to obtain co-cluster assignments, and a fixed input prior distribution to regularize co-clusters. 
Different from DeepCC, \sterling\ uses non-contrastive SSL to extract embeddings which can be used for various downstream tasks, and its loss function for co-clustering is simple and the joint distribution of different nodes is learned from data.

\section{Preliminary}
\textbf{Self-Supervised Learning on Bipartite Graphs.}
Given a bipartite graph $G=(U, V, E)$, where $U$, $V$ are disjoint node sets, and $E\subseteq U\times V$ is the edge set, the task is to extract informative node embeddings $\mathbf{U},\mathbf{V}\in\mathbb{R}^{|U|\times d}$ from $G$.
We use $u$, $v$ and $\mathbf{u}$, $\mathbf{v}$ to denote elements for $U$, $V$ and $\mathbf{U}$, $\mathbf{V}$.

\noindent
\textbf{Co-Clustering.}
Given a bipartite graph ${G}$, co-clustering maps $U$, $V$ into $N_K\ll|U|$, $N_L\ll|V|$ clusters via function $c_\theta$, which produces the probabilities of cluster assignments $p(k|u)$, $p(l|v)$ for nodes $u$, $v$.
Here $k$, $l$ are cluster indices.
We use $K$, $L$ to denote random variables of co-clusters. 

\section{Methodology}\label{sec:method}
\subsection{Overview of \sterling}
An overview of \sterling\ is shown in Fig. \ref{fig:overview}.
For local synergies, our idea is to obtain the node embeddings via an encoder $\mathcal{E}_\theta$ and maximize the similarity of positive node embedding pairs without minimizing the similarity of negative pairs (e.g., randomly sampled node pairs).
The positive pairs are selected based on both inter-type and intra-type synergies.
The inter-type $U$-$V$ positive pairs are the connected $u$-$v$ pairs ($e_{uv}\in E$).
The intra-type $U$-$U$ (or $V$-$V$) positive pairs are selected based on the k-NN $u$-$u$ (or $v$-$v$) pairs.
The pair ($\mathbf{u}_\theta$, $\mathbf{v}_\theta$) in Fig. 2 is an exemplar positive pair as they are connected in $G$: $e_{uv}\in E$.
However, directly maximizing the similarity of ($\mathbf{u}_\theta$, $\mathbf{v}_\theta$) without minimizing the similarity of negative pairs
may result in mode collapse (e.g., mapping all nodes to the same embedding) \cite{grill2020bootstrap}.
To address this issue, following \cite{grill2020bootstrap}, we use a target encoder $\mathcal{E}_\phi$ to obtain the bootstrapped target embeddings $(\mathbf{u}_\phi, \mathbf{v}_\phi)$, and also project the original online embeddings $(\mathbf{u}_\theta, \mathbf{v}_\theta)$ into $(\Tilde{\mathbf{u}}_\theta, \Tilde{\mathbf{v}}_\theta)$ via a projector $\mathcal{P}_\theta=(\mathcal{P}_\theta^U, \mathcal{P}_\theta^V)$, which could potentially add noise to the original embeddings $(\mathbf{u}_\theta, \mathbf{v}_\theta)$.
Then we maximize the similarity of $(\Tilde{\mathbf{u}}_\theta, \mathbf{v}_\phi)$ and $(\Tilde{\mathbf{v}}_\theta, \mathbf{u}_\phi)$.
We denote $f_\theta=\mathcal{P}_\theta\circ\mathcal{E}_\theta$ as the online local network and $f_\phi=\mathcal{E}_\phi$ as the target local network.
$f_\theta$ and $f_\phi$ are updated alternatively.
$f_\theta$ is trained by maximizing the similarity of positive node pairs $\mathcal{L}_{loc}$ and $f_\phi$ is updated via the Exponential Moving Average (EMA) of $\theta:\phi\leftarrow\tau\phi+(1-\tau)\theta,\tau\in[0,1]$.

For global synergies, we use a co-clustering network $g_\theta=\mathcal{C}_\theta=(\mathcal{C}^U_\theta, \mathcal{C}^V_\theta)$ to obtain the co-cluster probabilities $p(k|\Tilde{\mathbf{u}}_\theta)$, $p(l|\Tilde{\mathbf{v}}_\theta)$.
Since neural networks $\mathcal{P}$ are treated as an invective function in practice \cite{vincent2008extracting}, and thus $p(k|\Tilde{\mathbf{u}}_\theta)=p(k|\mathbf{u}_\theta)$, $p(l|\Tilde{\mathbf{v}}_\theta)=p(k|\Tilde{\mathbf{u}}_\theta)$.
The global objective $\mathcal{L}_{glb}$ is to maximize the mutual information of the co-clusters $K$ and $L$.
The global objective $\mathcal{L}_{glb}$ is also used to update $f_\theta$.
Note that co-clustering and mutual information calculation does not require negative node pairs.

After training, we use the online embeddings $\mathbf{u}_\theta$, $\mathbf{v}_\theta$ and co-cluster probabilities $p(k|\mathbf{u}_\theta)$, $p(l|\mathbf{v}_\theta)$ for downstream tasks.
$c_\theta=g_\theta\circ f_\theta$ is the final co-clustering function.

In the following content, we introduce local objective and global objective in detail.
Then we introduce the overall objective and provide theoretically analysis.

\subsection{Local Objective}
There are two kinds of local synergies among nodes in bipartite graphs: the explicit inter-type ($U$-$V$) and the implicit intra-type synergies ($U$-$U$, $V$-$V$).
The local objective $\mathcal{L}_{loc}$ captures both inter and intra-type synergies by maximizing the similarities of the inter and intra-type positive node pairs.

\noindent
\textbf{Inter-Type Synergies.}
Node embeddings $\mathbf{u}_\theta$ and $\mathbf{v}_\theta$ should be similar if $u$, $v$ are connected $e_{uv}\in E$.
Rather than directly maximizing the similarity of ($\mathbf{u}_\theta$,$\mathbf{v}_\theta$), we maximize the similarity of the projected online embedding $\Tilde{\mathbf{u}}_\theta$ and the target embedding $\mathbf{v}_\phi$, as well as $\Tilde{\mathbf{v}}_\theta$ and $\mathbf{u}_\phi$: 
\begin{equation}\label{eq:l_uv}
    \mathcal{L}_{uv} = -(\frac{\Tilde{\mathbf{u}}_\theta^T\mathbf{v}_\phi}{||\Tilde{\mathbf{u}}_\theta||\cdot||\mathbf{v}_\phi||} + \frac{\Tilde{\mathbf{v}}_\theta^T\mathbf{u}_\phi}{||\Tilde{\mathbf{v}}_\theta||\cdot||\mathbf{u}_\phi||})
\end{equation}
Note that when updating $\theta$ via the above objective function (as well as the following objective functions), 
$\phi$ is fixed.
$\phi$ is updated via EMA after $\theta$ is updated.
This practice could effectively avoid the mode collapse problem \cite{grill2020bootstrap}.

\noindent
\textbf{Intra-Type Synergies.}
If $u, u'\in U$ are highly correlated, then their embeddings should have a high similarity score\footnote{For clarity, we only use $U$-$U$ synergies to describe our method.
$V$-$V$ synergies are captured in the same way as $U$-$U$.}.
We determine the similarity of $u$, $u'$ from the perspectives of both graph structure and hidden semantics.

For the structure information, we use the Adamic-Adar (AA) index \cite{adamic2003friends}, which is the inverse log frequency (or degree) of the shared neighbors of $u$, $u'$:
\begin{equation}
    s_{aa}(u, u') = \sum_{v\in V_{uu'}}\frac{1}{\log(d_v)}
\end{equation}
where $V_{uu'}\subset V$ is the set of $v$ that connects with both $u$ and $u'$; $d_v$ is the degree of $v$. 
Essentially, the AA index calculates the second-order structure proximity and it has two characteristics:
(1) the more common neighbors $u$ and $u'$ share, the higher the AA index will be; (2) the lower the degree $d_v$ of the shared neighbor $v$, the higher the importance of $v$.

For the semantic information, given a node $u$, we determine its semantic similarity with another node $u'$ by:
\begin{equation}
    s_{emb}(u,u') = \frac{\mathbf{u}_\theta^T\mathbf{u}'_\phi}{||\mathbf{u}_\theta||\cdot||\mathbf{u}'_\phi||}
\end{equation}
The final similarity score between $u$ and $u'$ is thus:
\begin{equation}\label{eq:sim_score_full}
    s(u,u') = s_{aa}(u,u')\cdot s_{emb}(u,u')
\end{equation} Given $u\in U$, we select its k-Nearest-Neighbors (k-NN), i.e., top-K similar nodes, from $U$ to construct positive pairs.
Similar to $\mathcal{L}_{uv}$ in Equation \eqref{eq:l_uv}, the loss for each $u$ is:
\begin{equation}\label{eq:l_u}
    \mathcal{L}_{u} = \frac{-1}{N_{knn}}\sum_{u'\in \text{k-NN}(u)}(\frac{\Tilde{\mathbf{u}}_\theta^T\mathbf{u}'_\phi}{||\Tilde{\mathbf{u}}_\theta||\cdot||\mathbf{u}'_\phi||} +\frac{\Tilde{\mathbf{u}}^{'T}_\theta\mathbf{u}_\phi}{||\Tilde{\mathbf{u}}'_\theta||\cdot||\mathbf{u}_\phi||})
\end{equation}
where $N_{knn}$ is the number of the selected neighbors.
Note that for each $v\in V$, we use the same strategy as $u\in U$ described above to obtain $\mathcal{L}_v$:
\begin{equation}\label{eq:l_v}
    \mathcal{L}_{v} = \frac{-1}{N_{knn}}\sum_{v'\in \text{k-NN}(v)}(\frac{\Tilde{\mathbf{v}}_\theta^T\mathbf{v}'_\phi}{||\Tilde{\mathbf{v}}_\theta||\cdot||\mathbf{v}'_\phi||} +\frac{\Tilde{\mathbf{v}}^{'T}_\theta\mathbf{v}_\phi}{||\Tilde{\mathbf{v}}'_\theta||\cdot||\mathbf{v}_\phi||})
\end{equation}

\noindent
\textbf{Local Objective Function.}
During training, given a connected pair $(u,v)$, its local objective is:
\begin{equation}
    \mathcal{L}_{loc} = \lambda_{uv}\mathcal{L}_{uv} + \lambda_{u}\mathcal{L}_{u} + \lambda_{v}\mathcal{L}_{v}
\end{equation}
where 
$\mathcal{L}_{uv}$, $\mathcal{L}_u$ and $\mathcal{L}_v$ are obtained from Eq. \eqref{eq:l_uv}\eqref{eq:l_u}\eqref{eq:l_v}.

\subsection{Global Objective}
The two types of nodes $U$, $V$ are inherently correlated, and thus so as their clusters $K$, $L$, as illustrated by the green link in Fig. \ref{fig:example}.
Jointly co-clustering $U$ and $V$ usually produce better results than traditional one-side clustering \cite{xu2019deep}.
In this paper, we introduce a simple end-to-end co-clustering algorithm to capture the global cluster synergy by maximizing the mutual information of the co-clusters $I(K;L)$.
According to the definition of mutual information, to calculate $I(K;L)$, we need to obtain the joint distribution $p(k, l)$ and marginal distributions $p(k)$ and $p(l)$.
We decompose $p(k, l)$ by two other easy-to-obtain distributions $p(u,v)$ and $p(k,l|u,v)$ via $p(k, l)=\sum_{u, v}p(k, l|u, v)p(u, v)$.
In the following content, we first introduce the joint distribution $p(u,v)$, then introduce the conditional probability $p(k,l|u,v)$, and finally introduce the global objective $\mathcal{L}_{glb}$.

\noindent
\textbf{Joint Distribution  $p(u,v)$.}
The joint distribution $p(u, v)$ characterizes the connectivity of $u$, $v$.
Instead of simply deriving $p(u,v)$ from the edges $E$ and treating $p(u,v)$ as a fixed prior in \cite{xu2019deep}, \sterling\ learns $p(u,v)$ to encode both structure and semantic information.

For the structure information, we build an $n$-hop metapath \cite{dong2017metapath2vec} between $u$ and $v$ to find potential links between them.
The 1-hop $U$-$V$ metapath is the original $U$-$V$ graph, and the 2-hop $U$-$V$ metapath is the $U$-$V$-$U$-$V$ graph.
We denote $\mathbf{A}_{meta}$ as the adjacency matrix of the $n$-hop $u$-$v$ metapath.

For the semantic information, we construct $\mathbf{A}_{emb}$ from the extracted node embeddings:
\begin{equation}\label{eq:a_emb}
    \mathbf{A}_{emb} = \frac{1}{2}[\delta(\mathbf{U}_\theta\mathbf{V}_\phi^T) + \delta(\mathbf{U}_\phi\mathbf{V}_\theta^T)]
\end{equation}
where $\mathbf{U}_\theta$, $\mathbf{V}_\theta$ are online embeddings and $\mathbf{U}_\phi$, $\mathbf{V}_\phi$ are target embeddings, $\delta$ is an activation function.
We further filter out noisy connections by reseting the small values as 0:
\begin{equation}\label{eq:noise_filter}
    \mathbf{A}_{emb} = \max(\mathbf{A}_{emb}, \mu+\alpha\sigma)
\end{equation}
where $\mu$ and $\sigma$ are the mean and standard deviation of $\mathbf{A}_{emb}$, and $\alpha$ is a tunable threshold.

Finally, the joint distribution $p(U,V)$ is given by:
\begin{equation}\label{eq:p_u_v}
    p(U,V) = \frac{1}{Z}\mathbf{A}_{meta}\odot\mathbf{A}_{emb}
\end{equation}
where $Z$ is a normalization factor, $\odot$ is Hadamard product.

\noindent
\textbf{Conditional Distribution $p(k,l|u,v)$.}
As for $p(k, l|u, v)$, we obtain it via neural networks.
We first extract the online node embeddings $\mathbf{U}_\theta$ and $\mathbf{V}_\theta$ from the input bipartite graph ${G} = ({U}, {V}, {E})$ via $\mathcal{E}_\theta$.
Then we apply the function $\mathcal{C}_\theta\circ\mathcal{P}_\theta=(\mathcal{C}_\theta^U\circ\mathcal{P}_\theta^U,\mathcal{C}_\theta^V\circ\mathcal{P}_\theta^V)$ over $\mathbf{u}_\theta$ and $\mathbf{v}_\theta$ to obtain the co-cluster probabilities $p(k|\mathbf{u}_\theta)$ and $p(l|\mathbf{v}_\theta)$.
Since neural networks are usually treated as deterministic and injective functions in practice \cite{vincent2008extracting}, we have $p(k, l|u, v)=p(k, l|\mathbf{u}_\theta, \mathbf{v}_\theta)$.
Furthermore, since $\mathcal{C}_\theta^U\circ\mathcal{P}_\theta^U$ and $\mathcal{C}_\theta^V\circ\mathcal{P}_\theta^V$ have separate sets of parameters, it is natural to have $p(k, l|\mathbf{u}_\theta, \mathbf{v}_\theta)=p(k|\mathbf{u}_\theta)p(l|\mathbf{v}_\theta)$. 
Hence, we have $p(k, l|u, v)=p(k|\mathbf{u}_\theta)p(l|\mathbf{v}_\theta)$.

\noindent
\textbf{Global Objective Function.}
Combining the conditional distribution $p(k, l|u, v)=p(k|\mathbf{u}_\theta)p(l|\mathbf{v}_\theta)$ with the joint distribution $p(u, v)$ obtained from Equation \eqref{eq:p_u_v}, we have:
\begin{equation}\label{eq:pkl}
    p(k, l) = \sum_{u, v}p(k|\mathbf{u}_\theta)p(l|\mathbf{v}_\theta)p(u, v)
\end{equation}
Then we could easily obtain the marginal distributions $p(k)=\sum_{l}p(k, l)$, $p(l)=\sum_{k}p(k, l)$.
Since $p(k, l)$, $p(k)$ and $p(l)$ are calculated by neural networks, we can directly maximize $I(K;L)$.
The global loss $\mathcal{L}_{glb}$ is given by:
\begin{equation}
    \mathcal{L}_{glb}=-I(K;L)=-\sum_{k=1}^{N_K}\sum_{l=1}^{N_L}p(k,l)\log\frac{p(k,l)}{p(k)p(l)}
\end{equation}

\subsection{Overall Objective}\label{sec:overall_objective}
The overall objective function of \sterling\ is:
\begin{equation}
    \mathcal{L} = \mathcal{L}_{loc} + \mathcal{L}_{glb}
\end{equation}

\subsection{Theoretical Analysis}\label{sec:theory}
Theorem \ref{theorem:information_bound} shows that $I(\mathbf{U}_\theta;\mathbf{V}_\theta)$ is lower bounded by $I(K;L)$, indicating that maximizing $I(K;L)$ (or minimizing $\mathcal{L}_{glb}$) could improve the connectivity of $\mathbf{U}_\theta$ and $\mathbf{V}_\theta$ in the embedding space.
This theorem is corroborated by visualization results in Fig. \ref{fig:tsne_no_mi}-\ref{fig:tsne_full} in the Experiment section.
Please refer to Appendix for the proof.
\begin{theorem}[Information Bound]\label{theorem:information_bound}
The mutual information $I(\mathbf{U}_\theta;\mathbf{V}_\theta)$ of embeddings $\mathbf{U}_\theta$ and $\mathbf{V}_\theta$ is lower-bounded by the mutual information of co-clusters $I(K;L)$:
\begin{equation}\label{eq:theorem_bound}
    I(K;L)\leq I(\mathbf{U}_\theta;\mathbf{V}_\theta)
\end{equation}
\end{theorem}

\section{Experiments}
\subsection{Experimental Setup}\label{sec:setup}
\textbf{Data.} 
Table \ref{tab:data} shows the summary of datasets.
ML-100K and Wiki are processed by \cite{cao2021bipartite}, where Wiki has two splits (50\%/40\%) for training.
IMDB, Cornell and Citeceer are document-keyword bipartite graphs \cite{xu2019deep}.

\noindent
\textbf{Evaluation.}
For recommendation to a given user $u$, the score of an item $v$ is determined by the similarity of $\mathbf{u}_\theta$, $\mathbf{v}_\theta$, and then items with top-K scores are selected. 
We use F1, Normalized Discounted Cumulative Gain (NDCG), Mean Average Precision (MAP), and Mean Reciprocal Rank (MRR) as metrics.
For link prediction, given the learned embeddings $\mathbf{U}_\theta$, $\mathbf{V}_\theta$, and edges $E$, we train a logistic regression classifier, and then evaluate it on the test data. 
We use Area Under ROC Curves (AUC) as the metric.
For co-clustering, we assign the cluster with the highest probability as the cluster assignment for the given node, and use Normalized Mutual Information (NMI) and accuracy (ACC) as the metrics.

\begin{table}[t]
    \centering
    \footnotesize
    \setlength\tabcolsep{3pt} 
    {\begin{tabular}{ccccccc}
        \hline
        Dataset & Task & $|U|$ & $|V|$ & $|E|$ & \# Class\\
        \hline
        ML-100K & Recommendation & 943 & 1,682 & 100,000 & -\\
        \hline
        Wikipedia & Link Prediction & 15,000 & 3,214 & 64,095 & -\\
        \hline
        IMDB & Co-Clustering & 617 & 1878 & 20,156 & 17\\
        Cornell & Co-Clustering & 195 & 1,703 & 18,496 & 5\\
        Citeseer & Co-Clustreing & 3,312 & 3,703 & 105,165 & 6\\
        \hline
    \end{tabular}}
    \caption{Summary of the datasets.}
    \label{tab:data}
\end{table}

\noindent
\textbf{Baselines.}
Three groups of baselines are used: 
\textbf{(1) Bipartite Graph}: 
random-walk methods
BiNE \cite{gao2018bine}, PinSage \cite{ying2018graph};
matrix completion methods
GC-MC \cite{berg2017graph}, IGMC \cite{zhang2019inductive}; 
collaborative filtering methods
NeuMF \cite{he2017neural}, NGCF \cite{wang2019neural} 
DirectAU \cite{wang2022towards};
contrastive learning methods
BiGI \cite{cao2021bipartite}, SimGCL \cite{yu2022graph}, COIN \cite{jing2022coin};
\textbf{(2) Graph}:
traditional methods
DeepWalk \cite{perozzi2014deepwalk}, LINE \cite{tang2015line}, Node2vec \cite{grover2016node2vec}, VGAE \cite{kipf2016variational};
contrastive methods
GraphCL \cite{you2020graph}, 
non-contrastive methods
AFGRL \cite{lee2022augmentation}, GraphMAE \cite{hou2022graphmae}.
For heterogeneous graph methods, we compare with the random-walk method
Metapath2vec \cite{dong2017metapath2vec}, and contrastive methods
DMGI \cite{park2020unsupervised}, HDMI \cite{jing2021hdmi}, HeCo \cite{wang2021self}.
\textbf{(3) Co-Clustering}: 
traditional methods
CCInfo \cite{dhillon2003information},
SCC~\cite{dhillon2001co}, 
CCMod \cite{ailem2015co} and SCMK \cite{kang2017twin};
the SOTA deep learning method DeepCC \cite{xu2019deep}.

\begin{table*}[t]
    \centering
    \footnotesize
    \setlength\tabcolsep{0.8pt} 
    {\begin{tabular}{c|cccccccccc|cc}
        \hline
        & \multicolumn{10}{c|}{ML-100K} & Wiki(50\%) & {Wiki(40\%)}\\
        \hline
        Method & F1@10 & NDCG@3 & NDCG@5 & NDCG@10 & MAP@3 & MAP@5 & MAP@10 & MRR@3 & MRR@5 & MRR@10 & AUC & AUC\\
        \hline
        DeepWalk  & 14.20 & 7.17 & 9.32 & 13.13 & 2.72 & 3.54 & 4.92 & 43.86 & 46.83 & 48.75 & 87.19 & 81.60\\
        LINE  & 13.71 & 6.52 & 8.57 & 12.37 & 2.45 & 3.26 & 4.67 & 44.16 & 44.37 & 46.30 & 66.69 & 64.28 \\
        Node2vec & 14.13 & 7.69 & 9.91 & 13.41 & 3.07 & 3.90 & 5.19 & 44.80 & 48.02 & 49.78 & 89.37 & 88.41 \\
        VGAE & 11.38 & 6.43 & 8.18 & 10.93 & 2.35 & 2.95 & 3.94 & 39.39 & 42.32 & 43.68 & 87.81 & 86.32\\
        GraphCL & 19.46 & 10.13 & 13.24 & 18.17 & 4.17 & 5.65 & 8.04 & 58.04 & 60.67 & 61.97  & 94.40 & 93.67 \\
        GraphMAE & 21.28 & 11.35 & 14.67 & 20.11 & 4.66 & 6.22 & 9.01 & 62.96 & 65.12 & 66.12 & 95.12 & 94.48 \\
        AFGRL & 22.04 & 11.70 & 15.16 & 20.85 & 4.75 & 6.44 & 9.34 & 65.18 & 66.95 & 67.95 & 94.80 & 94.22 \\
        \hline
        Metapath2vec & 14.11 & 7.88 & 9.87 & 13.35 & 2.85 & 3.71 & 5.08 & 45.49 & 48.74 & 49.83 & 87.20 & 86.75\\
        DMGI & 19.58 & 10.16 & 13.13 & 18.31 & 3.98 & 5.33 & 7.82 & 59.33 & 61.37 & 62.71 & 93.02 & 92.01\\
        HDMI & 20.51 & 11.07 & 14.32 & 19.42 & 4.59 & 6.18 & 8.67 & 62.25 & 64.38 & 65.44 & 94.18  & 93.57 \\
        HeCo & 19.65 & 11.18 & 14.15 & 18.98 & 4.73 & 6.26 & 8.74 & 58.68 & 60.02 & 61.23 & 94.39 & 93.72\\
        \hline
        PinSage  & 21.68 & 10.95 & 14.51 & 20.27 & 4.52 & 6.18 & 9.13 & 62.56 & 64.77 & 65.76 & 94.27 & 92.79\\
        BiNE  & 14.83 & 7.69 & 9.96 & 13.79 & 2.87 & 3.80 & 5.24 & 48.14 & 50.94 & 52.51 & 94.33 & 93.15\\
        GC-MC & 20.65 & 10.88 & 13.87 & 19.21 & 4.41 & 5.84 & 8.43 & 60.60 & 62.21 & 63.53 & 91.90 & 91.40\\
        IGMC & 18.81 & 9.21 & 12.20 & 17.27 & 3.50 & 4.82 & 7.18 & 56.89 & 59.13 & 60.46 & 92.85 & 91.90\\
        NeuMF & 17.03 & 8.87 & 11.38 & 15.89 & 3.46 & 4.54 & 6.45 & 54.42 & 56.39 & 57.79 & 92.62 & 91.47\\
        NGCF & 21.64 & 11.03 & 14.49 & 20.29 & 4.49 & 6.15 & 9.11 & 62.56 & 64.62 & 65.55 & 94.26 & 93.06\\
        DirectAU & 21.04 & 11.13 & 14.27 & 19.65 & 4.76 & 6.21 & 8.79 & 59.99 & 62.53 & 63.80 & 94.62 & 93.98 \\
        BiGI & 23.36 & 12.50 & 15.92 & 22.14 & 5.41 & 7.15 & 10.50 & 66.01 & 67.70 & 68.78 & 94.91 & 94.08\\
        COIN & 24.78 & 13.48 & 17.37 & 23.62 & 5.71 & 7.82 & 11.34 & 70.58 & 72.14 & 72.76 & 95.30 & 94.53\\
        SimGCL & 25.19 & 13.51 & 17.62 & 24.08 & 5.73 & 7.94 & 11.62 & 71.31 & 73.02 & 73.77 & 95.22 & 94.62\\
        \hline
        \sterling\ & \textbf{25.54} & \textbf{14.01} & \textbf{18.23} & \textbf{24.37} & \textbf{6.06} & \textbf{8.40} & \textbf{11.93} & \textbf{71.99} & \textbf{73.55} & \textbf{74.27} & \textbf{95.48} & \textbf{95.04}\\
        \hline
    \end{tabular}}
    \caption{Performance (\%) of top-K recommendation on ML-100K (left), and link prediction on Wikipedia (right).}
    \label{tab:ml_100k}
\end{table*}

\begin{table}[t]
    \centering
    \footnotesize
    \setlength\tabcolsep{1pt} 
    {\begin{tabular}{c|c|ccccccc}
        \hline
        Metric & Dataset & SCC  & CCMod & CCInfo  & SCMK  & DeepCC & \sterling\\
        \hline
        & IMDB & 25.5 & 21.6 & 18.7 & 18.4 & 26.8 & \textbf{33.4}\\
        NMI &Cornell & 28.8 & 18.9 & 20.6 & 25.7 & 35.4 & \textbf{37.5} \\
        & Citeseer & 15.2 & 16.9 & 17.7 & 21.1 & 29.8 & \textbf{31.6} \\
        \hline
        \hline
        & IMDB & 25.2 & 24.7 & 23.0 & 18.4 & 23.3 & \textbf{34.7}\\
        ACC & Cornell & 58.9 & 55.5 & 56.6 & 49.6 & 68.7 & \textbf{73.4} \\
        & Citeseer & 37.4 & 44.7 & 43.0 & 50.2 & 59.3 & \textbf{63.7} \\
        \hline
    \end{tabular}}
    \caption{NMI (upper) \& ACC (lower) for co-clustering.}
    \label{tab:co_clustering_nmi}
\end{table}

\noindent
\textbf{Implementation.}
The encoder $\mathcal{E}$ is a simple $L$-layer message passing model $\mathbf{u}^{(l+1)}=\text{AGG}(\mathbf{u}^{(l)},\{\mathbf{v}^{(l)}:e_{uv}\in E\})$, where AGG is an aggreagation function.
$\mathbf{{v}}^{(l+1)}$ is obtained in a similar way.
The projector $\mathcal{P}$ is either a Multi-Layer Perceptron (MLP) or identity mapping.
The cluster network $\mathcal{C}$ is a MLP, and its final activation function is softmax.
We set $N_K=N_L$ for co-clusters.
We perform grid search over several hyper-parameters such as $N_{knn}$, $N_K$, $\alpha$, the number of layers, and embedding size $d$.
We set $\delta$ as absolute activation.
Please refer to Appendix for more details.

\begin{table}[t]
    \centering
    \footnotesize
    \setlength\tabcolsep{2pt} 
    {\begin{tabular}{l|ccc}
        \hline
        \multirow{2}{*}{Method} & ML-100K & Wiki(40\%) & Cornell\\
                                & (F1@10) & (AUC) & (NMI)\\
        \hline
        \sterling\ & \textbf{25.54} & \textbf{95.04} & \textbf{37.51}\\
        \hline
        w/o $\mathcal{L}_{glb}$ & 25.12 & 94.37 & 26.02 \\
        w/o $\mathcal{L}_{v}$ & 21.67 & 94.89 & 37.06 \\
        w/o $\mathcal{L}_{u}$ & 25.38 & 94.74 & 35.74 \\
        w/o $\mathcal{L}_{uv}$ & 0.20 & 94.76 & 30.93 \\
        \hline
        w/o $\mathbf{A}_{meta}$ & 25.41 & 95.01 &  26.61\\
         w/o $\mathbf{A}_{emb}$  & 25.15 & 94.42 & 37.43 \\
        w/o noise filter(Eq.\eqref{eq:noise_filter}) & 25.17 & 94.70 & 36.82 \\
        abs $\rightarrow$ ReLU(Eq.\eqref{eq:a_emb}) & 25.20 & 94.84 & 36.53 \\
        \hline
        w/o $s_{aa}$ & 20.32 & 94.23 & 32.75 \\
        w/o $s_{emb}$ & 24.84 & 94.84 & 37.10 \\
        AA $\rightarrow$ Co-HITS & 24.94 & 94.90 & 35.87 \\
        \hline
    \end{tabular}}
    \caption{Ablation study.}
    \label{tab:ablation}
\end{table}

\subsection{Overall Performance}
\textbf{Recommendation.}
The results on ML-100K are shown in the left part of Table \ref{tab:ml_100k}.
The upper/middle/lower groups of baselines are homogeneous/heterogeneous/bipartite methods respectively.
Comparing the best-performing baselines of the three groups, we observe that (1) the contrasive homogeneous and heterogeneous methods (GraphCL and HDMI) are competitive to each other, and non-contrastive method AFGRL performs better than contrastive methods;
(2) the bipartite graph method SimGCL is significantly better than GraphCL/AFGRL/HDMI.
This observation demonstrates that homogeneous and heterogeneous graph methods perform sub-optimally on bipartite graphs as they do not capture the synergies of bipartite graphs.
\sterling\ further outperforms SimGCL over all metrics, indicating the superiority of the non-contrastive SSL approach over the contrastive approaches for bipartite graphs.

\noindent
\textbf{Link Prediction.}
The results are shown in the right part of Table \ref{tab:ml_100k}, where the upper/middle/lower parts are homogeneous/heterogeneous/bipartite methods.
We could observe:
(1) \sterling\ has the best overall performance;
(2) homogeneous/heterogeneous approaches are sub-optimal on bipartite graphs;
(3) the non-contrastive method (\sterling) is better than contrastive method (COIN/SimGCL).


\noindent
\textbf{Co-Clustering.}
The results in Tables \ref{tab:co_clustering_nmi} show that among all the baseline methods, DeepCC achieves the highest NMI and ACC scores on all datasets.
DeepCC obtains node embeddings by reconstructing the input matrix, and trains cluster networks based on the fixed prior distribution derived from the input matrix.
\sterling\ has further improvements over DeepCC, indicating that (1) the non-contrastive SSL is better than the re-construction based representation learning, and (2) the learned joint distribution $p(U;V)$ helps improve the performance over the fixed prior distribution.

\subsection{Ablation Study}
\textbf{Components in $\mathcal{L}$.}
In the upper part of Table \ref{tab:ablation}, we study the impact of each component in the loss function.
(1) The global synergies $\mathcal{L}_{glb}$ are important for all datasets.
Note that for the Cornell dataset (co-clustering task), removing $\mathcal{L}_{glb}$ means we use a dummy un-trained $\mathcal{C}_\theta$, which is a random deterministic function mapping similar node embeddings to similar cluster distributions.
An NMI score of 26.02 rather than 0 means that the local loss $\mathcal{L}_{loc}$ can discover the clusters of node embeddings to a certain degree.
(2) The intra-types synergies $\mathcal{L}_u$ and $\mathcal{L}_v$ have different impacts for different datasets.
For ML-100K, $\mathcal{L}_v$ is more important, and for Wiki and Cornell, $\mathcal{L}_u$ has a higher impact.
(3) The inter-type synergies $\mathcal{L}_{uv}$ are indispensable for all datasets. 
Surprisingly, for ML-100K (recommendation task) the model can barely recommend correct items to given users. 
These results imply that for simpler tasks, which only require class/cluster-level predictions, such as link prediction (0/1 classification) and co-clustering, the implicit intra-type synergies $\mathcal{L}_{u}$, $\mathcal{L}_{v}$ and global synergies could provide a large amount of the information needed.
For harder tasks requiring precise element-level predictions, e.g., recommendation (ranking items for a given users), explicit inter-type information $\mathcal{L}_{uv}$ is required.

\begin{figure*}[t]
\centering
\begin{subfigure}[b]{.2\textwidth}
\includegraphics[width=\linewidth]{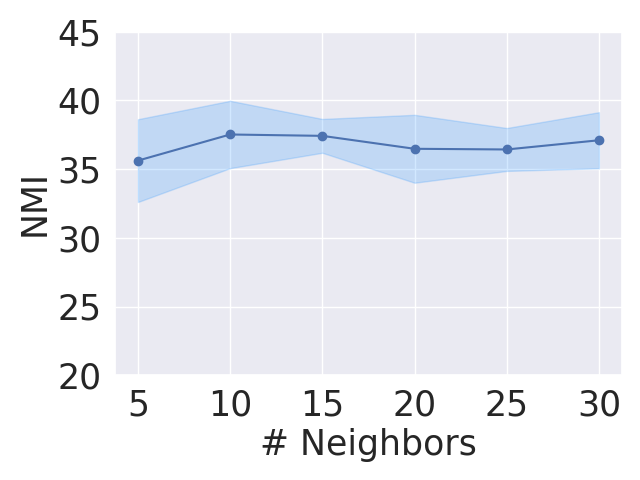}
 \caption{$N_{knn}$}\label{fig:cornell_topk}
\end{subfigure}\quad\quad
\begin{subfigure}[b]{.2\textwidth}
  \includegraphics[width=\linewidth]{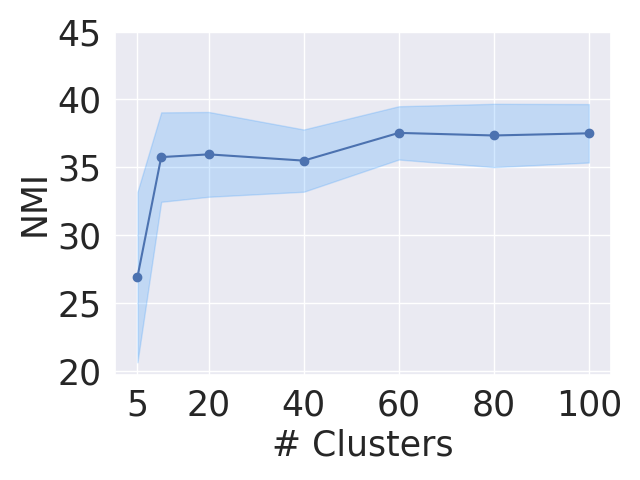}
  \caption{$N_K=N_L$}\label{fig:cornell_cluster}
\end{subfigure}\quad\quad
\begin{subfigure}[b]{.2\textwidth}
  \includegraphics[width=\linewidth]{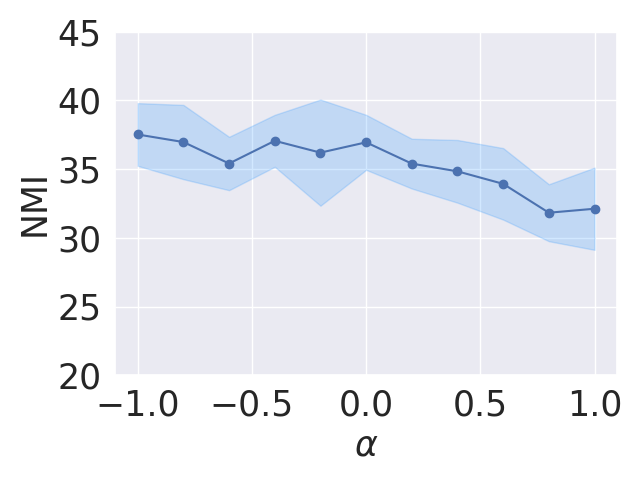}
  \caption{$\alpha$}\label{fig:cornell_alpha}
\end{subfigure}\quad\quad
\begin{subfigure}[b]{.2\textwidth}
  \includegraphics[width=\linewidth]{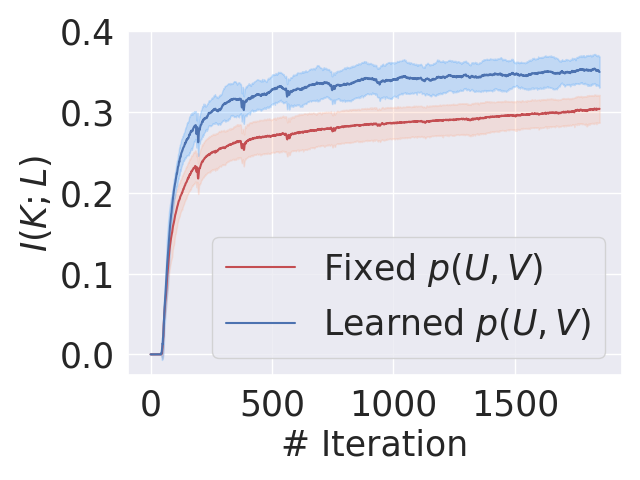}
  \caption{Convergence of $I(K;L)$}\label{fig:cornell_mi}
\end{subfigure}
\caption{(a-c) Sensitivity analysis and (d) convergence of $I(K;L)$ on the Cornell dataset.}
\label{fig:sensitivity}
\end{figure*}
\begin{figure*}
\centering
\begin{subfigure}[b]{.17\textwidth}
\includegraphics[width=\linewidth]{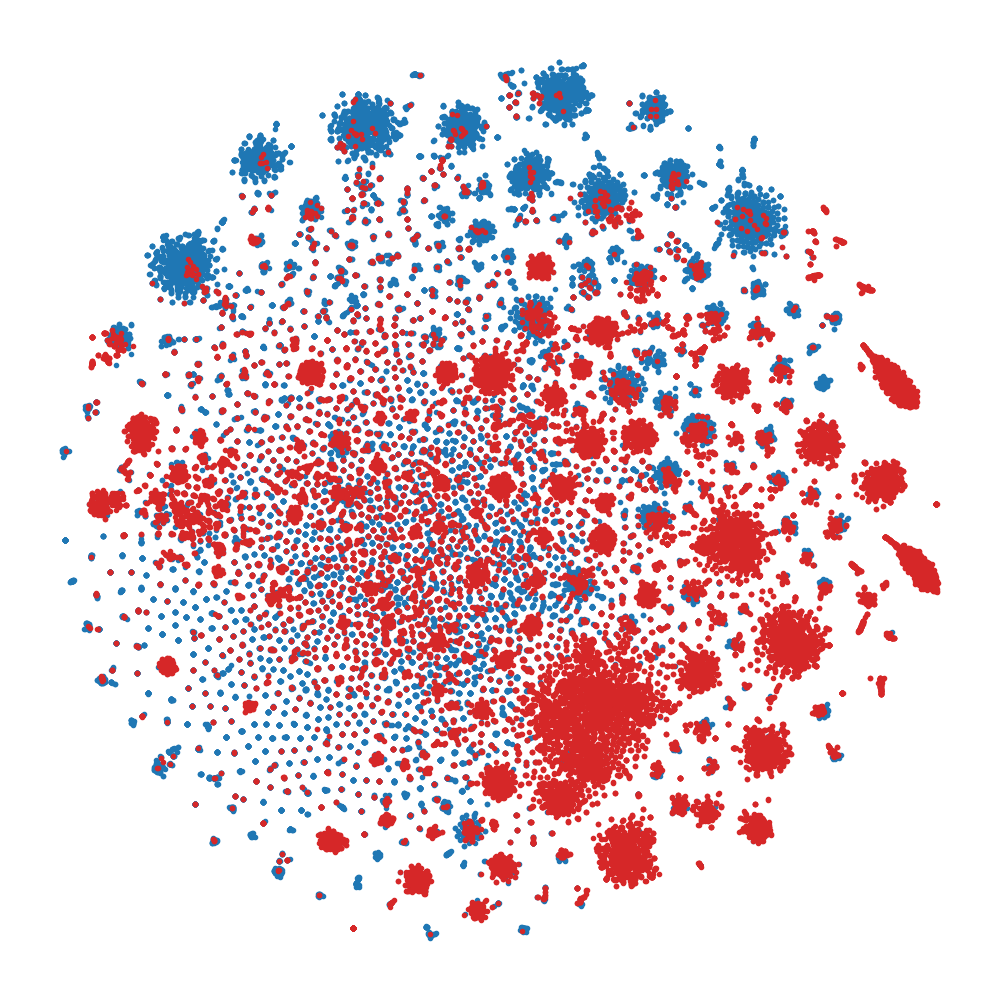}
 \caption{$\mathcal{L}_{loc}$}\label{fig:tsne_no_mi}
\end{subfigure}\quad
\begin{subfigure}[b]{.17\textwidth}
  \includegraphics[width=\linewidth]{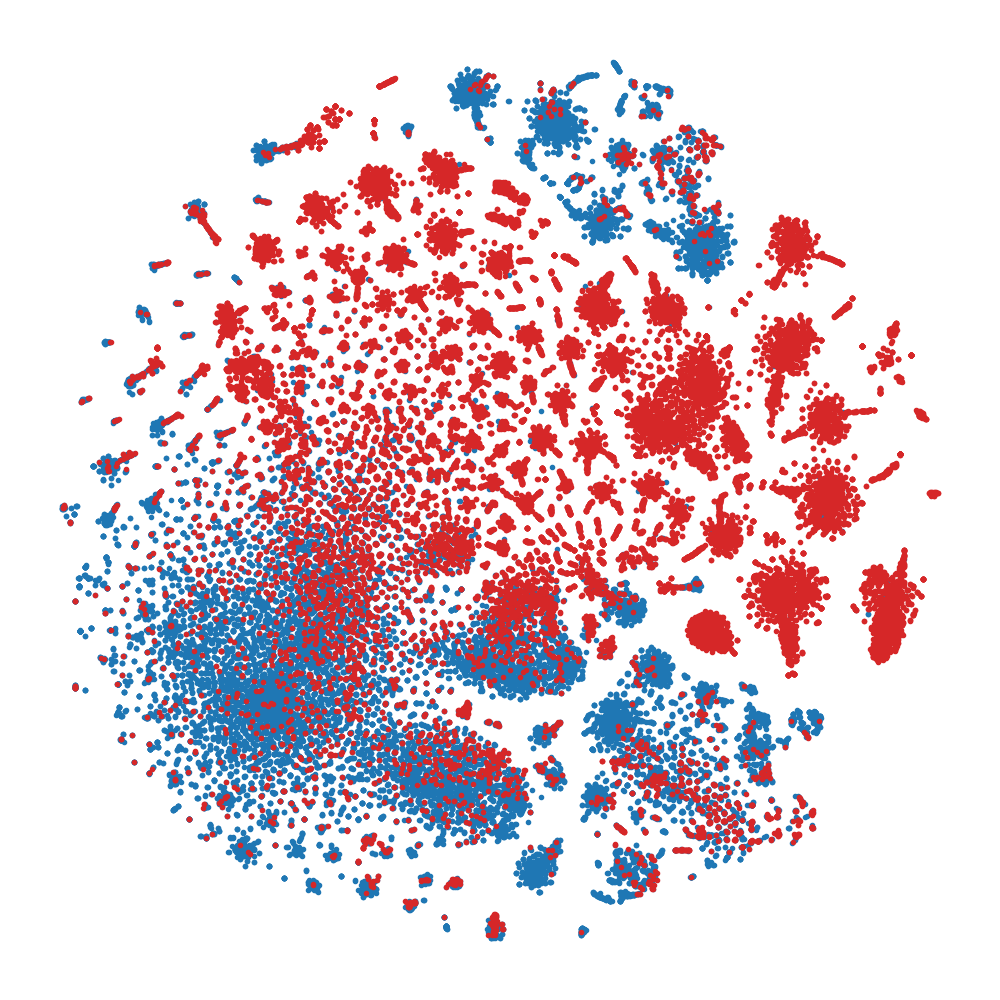}
  \caption{$\mathcal{L}_{loc}+\mathcal{L}_{glb}$}\label{fig:tsne_full}
\end{subfigure}\quad
\begin{subfigure}[b]{.28\textwidth}
\includegraphics[width=\linewidth]{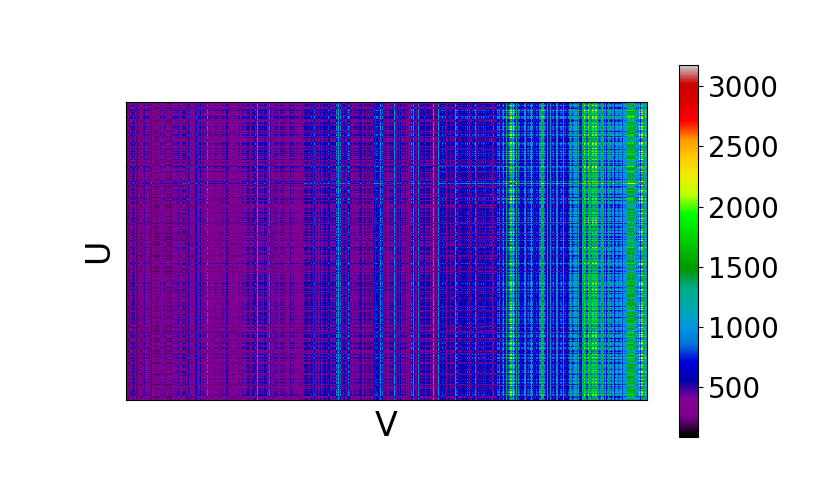}
 \caption{Before filtering.}\label{fig:no_filter}
\end{subfigure}
\begin{subfigure}[b]{.28\textwidth}
  \includegraphics[width=\linewidth]{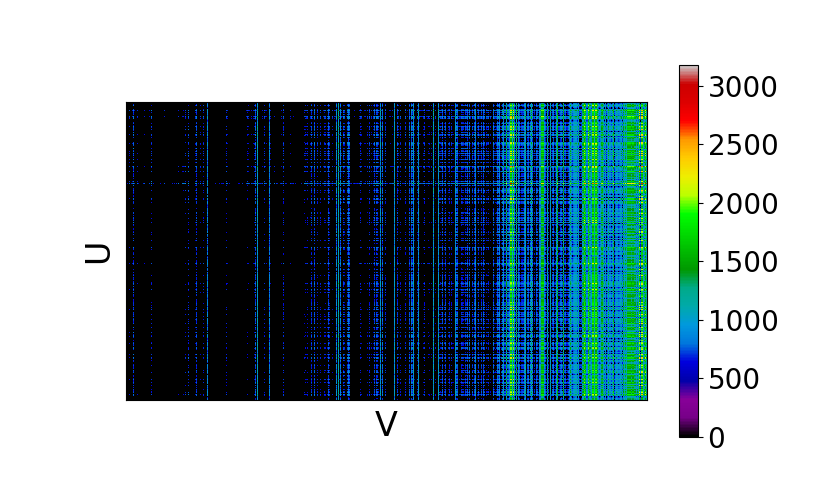}
  \caption{After filtering.}\label{fig:filter}
\end{subfigure}
\caption{(a-b) T-SNE visualization on Wiki (40\%). (c-d) Visualization of noise filter on ML-100K.}
\label{fig:visualization}
\end{figure*}

\noindent\textbf{Components in $\mathcal{L}_{glb}$.}
The results in the middle part of Table \ref{tab:ablation} show that:
(1) $\mathbf{A}_{meta}$ has more impact for Cornell and $\mathbf{A}_{emb}$ is more influential for ML-100K and Wiki; 
(2) noise filtering is useful;
(3) surprisingly, the absolute activation performs better than ReLU.
We believe this is because the datasets only record strength of connections but do not distinguish the sign (positive/negative) of connections.

\noindent
\textbf{Components in $\mathcal{L}_{loc}$.}
The results in the lower part of Table \ref{tab:ablation} show that
(1) both $s_{aa}$, $s_{emb}$ are crucial, and $s_{aa}$ plays a more important role;
(2) the AA index is better than another popular index Co-HITS \cite{deng2009generalized} since the AA index will raise the weights of the low-degree neighbors, and lower the weights of the high-degree neighbors.

\begin{table}[t]
    \centering
    \footnotesize
    \setlength\tabcolsep{2pt} 
    {\begin{tabular}{l|ccc}
        \hline
        \multirow{2}{*}{Method} & ML-100K & Wiki(40\%) & Cornell\\
        & (F1@10) & (AUC) & (NMI)\\
        \hline
        \sterling\ & \textbf{25.54} & \textbf{95.04} & \textbf{37.51}\\
        \hline
        $\mathbf{A}_{emb}=\mathbf{U}_\theta\mathbf{V}_\phi^T$ & 25.37 & 94.48 & 36.92\\
        $\mathbf{A}_{emb}=\mathbf{U}_\phi\mathbf{V}_\theta^T$ & 25.46 & 94.91 & 35.01 \\
        $\mathbf{A}_{emb}=\mathbf{U}_\theta\mathbf{V}_\theta^T$ & 25.49 & 94.71 & 36.57\\
        $\mathbf{A}_{emb}=\mathbf{U}_\phi\mathbf{V}_\phi^T$ & 25.46 & 94.36 & 36.09\\
        \hline
        $s_{emb}=\frac{1}{2}(\mathbf{u}_\theta^T\mathbf{v}_\phi + \mathbf{u}_\phi^T\mathbf{v}_\theta)$ & 25.40 & 94.97 & 35.84 \\
        $s_{emb}=\mathbf{u}_\phi^T\mathbf{v}_\theta$ & 25.46 & 94.93 & 36.82\\
        $s_{emb}=\mathbf{u}_\theta^T\mathbf{v}_\theta$ & 25.34 & 95.03 & 37.34\\
        $s_{emb}=\mathbf{u}_\phi^T\mathbf{v}_\phi$ & 25.43 & 94.86 & 35.13\\
        \hline
    \end{tabular}}
    \caption{Different Variants of $\mathbf{A}_{emb}$ and $s_{emb}$.}
    \label{tab:ablation_other}
\end{table}

\subsection{Other Results}
\noindent
\textbf{Sensitivity.}
The sensitivity analysis on the Cornell dataset is shown in Fig. \ref{fig:cornell_topk}-\ref{fig:cornell_alpha}.
The optimal numbers of k-NN, co-clusters and noise threshold are: $N_{knn}\approx10$, $N_K=N_L\geq60$, and $\alpha\in[-1, 0]$.
The optimal $N_K$ is larger than the real number of classes (5), implying over-clustering is beneficial.

\noindent
\textbf{Convergence of MI.}
Fig. \ref{fig:cornell_mi} shows the $I(K;L)$ of each training iteration on the Cornell dataset.
$I(K;L)$ will converge, using either the fixed or learned $p(U,V)$.
However, the learned $p(U,V)$ results in higher $I(K;L)$ in the end.

\noindent
\textbf{Visualization of Emb.} T-SNE \cite{van2008visualizing} visualization for embeddings of Wiki(40\%) test data is shown in Fig. \ref{fig:tsne_no_mi}-\ref{fig:tsne_full}.
Each embedding is the concatenation of a given ($\mathbf{u}_\theta$, $\mathbf{v}_\theta$) pair, which is the input of the logistic regression classifier.
These embeddings are labeled with 1/0, indicating whether a pair is true/false.
The two colors in Fig. \ref{fig:tsne_no_mi}-\ref{fig:tsne_full} correspond to the two classes.
Fig. \ref{fig:tsne_no_mi}-\ref{fig:tsne_full} show that $\mathcal{L}_{glb}$ further helps discover the underlying connectivity of ($\mathbf{u}_\theta$, $\mathbf{v}_\theta$) and better separate embeddings of the two classes than $\mathcal{L}_{loc}$ alone, which corroborates Theorem \ref{theorem:information_bound}.

\noindent
\textbf{Visualization of Noise Filtering.}
We visualize the effect of noise filtering (Eq.\eqref{eq:noise_filter}) on ML-100K in Fig. \ref{fig:no_filter}-\ref{fig:filter}.
The weak connections can be effectively filtered out since the purple dots in Fig. \ref{fig:no_filter} are removed (becomes black) after filtering in Fig. \ref{fig:filter}.
The matrix densities (the ratio of non-zero elements) in Fig. \ref{fig:no_filter} and \ref{fig:filter} are 100\% and 35.66\% respectively.

\noindent
\textbf{Different Variants.}
We show the results of different variants of $\mathbf{A}_{emb}$ and $s_{emb}$ (normalization is dropped for clarity) in Table \ref{tab:ablation_other}, which indicate that the variants used in \sterling\ has the best overall performance.

\section{Conclusion}
In this paper, we introduce a novel non-contrastive SSL method \sterling\ for bipartite graphs, which preserves both local inter/intra-type synergies and global co-cluster synergies.
Theoretical analysis indicates that \sterling\ could improve the connectivity of the two node types in the embedding space.
Empirical evaluation shows that the node embeddings extracted by \sterling\ have SOTA performance on various downstream tasks.

\section{Acknowledgments}
This work is partly supported by NSF (\#2229461),
DARPA (HR001121C0165), 
NIFA (2020-67021-32799), 
MIT-IBM Watson AI Lab, 
and IBM-Illinois Discovery Accelerator Institute. 
The content of the information in this document does not necessarily reflect the position or the policy of the Government or Amazon, and no official endorsement should be inferred.  The U.S. Government is authorized to reproduce and distribute reprints for Government purposes notwithstanding any copyright notation here on.
Dr. Chanyoung Park is supported by the IITP grant funded by the Korea government (MSIT) (RS-2023-00216011).

\bibliography{aaai24}

\begin{thebibliography}{77}
\providecommand{\natexlab}[1]{#1}

\bibitem[{Adamic and Adar(2003)}]{adamic2003friends}
Adamic, L.~A.; and Adar, E. 2003.
\newblock Friends and neighbors on the web.
\newblock \emph{Social networks}.

\bibitem[{Ailem, Role, and Nadif(2015)}]{ailem2015co}
Ailem, M.; Role, F.; and Nadif, M. 2015.
\newblock Co-clustering document-term matrices by direct maximization of graph
  modularity.
\newblock In \emph{CIKM}.

\bibitem[{Berg, Kipf, and Welling(2017)}]{berg2017graph}
Berg, R. v.~d.; Kipf, T.~N.; and Welling, M. 2017.
\newblock Graph convolutional matrix completion.
\newblock \emph{arXiv preprint arXiv:1706.02263}.

\bibitem[{Cao et~al.(2021)Cao, Lin, Guo, Liu, Liu, and Wang}]{cao2021bipartite}
Cao, J.; Lin, X.; Guo, S.; Liu, L.; Liu, T.; and Wang, B. 2021.
\newblock Bipartite graph embedding via mutual information maximization.
\newblock In \emph{WSDM}.

\bibitem[{Deng, Lyu, and King(2009)}]{deng2009generalized}
Deng, H.; Lyu, M.; and King, I. 2009.
\newblock A generalized co-hits algorithm and its application to bipartite
  graphs.
\newblock In \emph{KDD}.

\bibitem[{Dhillon(2001)}]{dhillon2001co}
Dhillon, I.~S. 2001.
\newblock Co-clustering documents and words using bipartite spectral graph
  partitioning.
\newblock In \emph{KDD}.

\bibitem[{Dhillon, Mallela, and Modha(2003)}]{dhillon2003information}
Dhillon, I.~S.; Mallela, S.; and Modha, D.~S. 2003.
\newblock Information-theoretic co-clustering.
\newblock In \emph{KDD}.

\bibitem[{Ding et~al.(2023)Ding, Wang, Yang, and Liu}]{ding2023eliciting}
Ding, K.; Wang, Y.; Yang, Y.; and Liu, H. 2023.
\newblock Eliciting structural and semantic global knowledge in unsupervised
  graph contrastive learning.
\newblock In \emph{AAAI}.

\bibitem[{Dong, Chawla, and Swami(2017)}]{dong2017metapath2vec}
Dong, Y.; Chawla, N.~V.; and Swami, A. 2017.
\newblock metapath2vec: Scalable representation learning for heterogeneous
  networks.
\newblock In \emph{KDD}.

\bibitem[{Feng et~al.(2022{\natexlab{a}})Feng, Jing, Zhu, and
  Tong}]{feng2022adversarial}
Feng, S.; Jing, B.; Zhu, Y.; and Tong, H. 2022{\natexlab{a}}.
\newblock Adversarial graph contrastive learning with information
  regularization.
\newblock In \emph{TheWebConf}.

\bibitem[{Feng et~al.(2022{\natexlab{b}})Feng, Jing, Zhu, and
  Tong}]{feng2022ariel}
Feng, S.; Jing, B.; Zhu, Y.; and Tong, H. 2022{\natexlab{b}}.
\newblock ARIEL: Adversarial Graph Contrastive Learning.
\newblock \emph{arXiv preprint arXiv:2208.06956}.

\bibitem[{Fu et~al.(2020)Fu, Xu, Li, Tong, and He}]{fu2020view}
Fu, D.; Xu, Z.; Li, B.; Tong, H.; and He, J. 2020.
\newblock A view-adversarial framework for multi-view network embedding.
\newblock In \emph{CIKM}.

\bibitem[{Gao et~al.(2018)Gao, Chen, He, and Zhou}]{gao2018bine}
Gao, M.; Chen, L.; He, X.; and Zhou, A. 2018.
\newblock Bine: Bipartite network embedding.
\newblock In \emph{SIGIR}.

\bibitem[{Goyal and Ferrara(2018)}]{goyal2018graph}
Goyal, P.; and Ferrara, E. 2018.
\newblock Graph embedding techniques, applications, and performance: A survey.
\newblock \emph{Knowledge-Based Systems}.

\bibitem[{Grill et~al.(2020)Grill, Strub, Altch{\'e}, Tallec, Richemond,
  Buchatskaya, Doersch, Avila~Pires, Guo, Gheshlaghi~Azar
  et~al.}]{grill2020bootstrap}
Grill, J.-B.; Strub, F.; Altch{\'e}, F.; Tallec, C.; Richemond, P.;
  Buchatskaya, E.; Doersch, C.; Avila~Pires, B.; Guo, Z.; Gheshlaghi~Azar, M.;
  et~al. 2020.
\newblock Bootstrap your own latent-a new approach to self-supervised learning.
\newblock \emph{NeurIPS}.

\bibitem[{Grover and Leskovec(2016)}]{grover2016node2vec}
Grover, A.; and Leskovec, J. 2016.
\newblock node2vec: Scalable feature learning for networks.
\newblock In \emph{KDD}.

\bibitem[{He et~al.(2020)He, Deng, Wang, Li, Zhang, and Wang}]{he2020lightgcn}
He, X.; Deng, K.; Wang, X.; Li, Y.; Zhang, Y.; and Wang, M. 2020.
\newblock Lightgcn: Simplifying and powering graph convolution network for
  recommendation.
\newblock In \emph{SIGIR}.

\bibitem[{He et~al.(2016)He, Gao, Kan, and Wang}]{he2016birank}
He, X.; Gao, M.; Kan, M.-Y.; and Wang, D. 2016.
\newblock Birank: Towards ranking on bipartite graphs.
\newblock \emph{TKDE}.

\bibitem[{He et~al.(2017)He, Liao, Zhang, Nie, Hu, and Chua}]{he2017neural}
He, X.; Liao, L.; Zhang, H.; Nie, L.; Hu, X.; and Chua, T.-S. 2017.
\newblock Neural collaborative filtering.
\newblock In \emph{TheWebConf}.

\bibitem[{Hou et~al.(2022)Hou, Liu, Cen, Dong, Yang, Wang, and
  Tang}]{hou2022graphmae}
Hou, Z.; Liu, X.; Cen, Y.; Dong, Y.; Yang, H.; Wang, C.; and Tang, J. 2022.
\newblock Graphmae: Self-supervised masked graph autoencoders.
\newblock In \emph{KDD}.

\bibitem[{Hu et~al.(2020)Hu, Dong, Wang, and Sun}]{hu2020heterogeneous}
Hu, Z.; Dong, Y.; Wang, K.; and Sun, Y. 2020.
\newblock Heterogeneous graph transformer.
\newblock In \emph{TheWebConference}.

\bibitem[{Huang et~al.(2020)Huang, Li, Fang, Fan, and Yang}]{huang2020biane}
Huang, W.; Li, Y.; Fang, Y.; Fan, J.; and Yang, H. 2020.
\newblock Biane: Bipartite attributed network embedding.
\newblock In \emph{SIGIR}.

\bibitem[{Jiang, Huang, and Huang(2023)}]{jiang2023adaptive}
Jiang, Y.; Huang, C.; and Huang, L. 2023.
\newblock Adaptive graph contrastive learning for recommendation.
\newblock In \emph{KDD}.

\bibitem[{Jing et~al.(2022{\natexlab{a}})Jing, Feng, Xiang, Chen, Chen, and
  Tong}]{jing2022x}
Jing, B.; Feng, S.; Xiang, Y.; Chen, X.; Chen, Y.; and Tong, H.
  2022{\natexlab{a}}.
\newblock X-GOAL: Multiplex Heterogeneous Graph Prototypical Contrastive
  Learning.
\newblock In \emph{CIKM}.

\bibitem[{Jing, Park, and Tong(2021)}]{jing2021hdmi}
Jing, B.; Park, C.; and Tong, H. 2021.
\newblock Hdmi: High-order deep multiplex infomax.
\newblock In \emph{TheWebConf}.

\bibitem[{Jing, Tong, and Zhu(2021)}]{jing2021network}
Jing, B.; Tong, H.; and Zhu, Y. 2021.
\newblock Network of tensor time series.
\newblock In \emph{TheWebConf}.

\bibitem[{Jing et~al.(2022{\natexlab{b}})Jing, Yan, Zhu, and
  Tong}]{jing2022coin}
Jing, B.; Yan, Y.; Zhu, Y.; and Tong, H. 2022{\natexlab{b}}.
\newblock COIN: Co-Cluster Infomax for Bipartite Graphs.
\newblock \emph{NeurIPS GLFrontiers}.

\bibitem[{Jing et~al.(2021)Jing, You, Yang, Fan, and Tong}]{jing2021multiplex}
Jing, B.; You, Z.; Yang, T.; Fan, W.; and Tong, H. 2021.
\newblock Multiplex Graph Neural Network for Extractive Text Summarization.
\newblock In \emph{EMNLP}.

\bibitem[{Kang, Peng, and Cheng(2017)}]{kang2017twin}
Kang, Z.; Peng, C.; and Cheng, Q. 2017.
\newblock Twin learning for similarity and clustering: A unified kernel
  approach.
\newblock In \emph{AAAI}.

\bibitem[{Kipf and Welling(2016)}]{kipf2016variational}
Kipf, T.~N.; and Welling, M. 2016.
\newblock Variational graph auto-encoders.
\newblock \emph{arXiv preprint arXiv:1611.07308}.

\bibitem[{Lee, Lee, and Park(2022)}]{lee2022augmentation}
Lee, N.; Lee, J.; and Park, C. 2022.
\newblock Augmentation-free self-supervised learning on graphs.
\newblock In \emph{AAAI}.

\bibitem[{Li, Jing, and Tong(2022)}]{li2022graph}
Li, B.; Jing, B.; and Tong, H. 2022.
\newblock Graph Communal Contrastive Learning.
\newblock In \emph{TheWebConf}.

\bibitem[{Mao et~al.(2021)Mao, Zhu, Xiao, Lu, Wang, and He}]{mao2021ultragcn}
Mao, K.; Zhu, J.; Xiao, X.; Lu, B.; Wang, Z.; and He, X. 2021.
\newblock UltraGCN: ultra simplification of graph convolutional networks for
  recommendation.
\newblock In \emph{CIKM}.

\bibitem[{Oord, Li, and Vinyals(2018)}]{oord2018representation}
Oord, A. v.~d.; Li, Y.; and Vinyals, O. 2018.
\newblock Representation learning with contrastive predictive coding.
\newblock \emph{arXiv preprint arXiv:1807.03748}.

\bibitem[{Park et~al.(2020)Park, Kim, Han, and Yu}]{park2020unsupervised}
Park, C.; Kim, D.; Han, J.; and Yu, H. 2020.
\newblock Unsupervised attributed multiplex network embedding.
\newblock In \emph{AAAI}.

\bibitem[{Pavlopoulos et~al.(2018)Pavlopoulos, Kontou, Pavlopoulou, Bouyioukos,
  Markou, and Bagos}]{pavlopoulos2018bipartite}
Pavlopoulos, G.~A.; Kontou, P.~I.; Pavlopoulou, A.; Bouyioukos, C.; Markou, E.;
  and Bagos, P.~G. 2018.
\newblock Bipartite graphs in systems biology and medicine: a survey of methods
  and applications.
\newblock \emph{GigaScience}, 7.

\bibitem[{Perozzi, Al-Rfou, and Skiena(2014)}]{perozzi2014deepwalk}
Perozzi, B.; Al-Rfou, R.; and Skiena, S. 2014.
\newblock Deepwalk: Online learning of social representations.
\newblock In \emph{KDD}.

\bibitem[{Shafiei and Milios(2006)}]{shafiei2006latent}
Shafiei, M.~M.; and Milios, E.~E. 2006.
\newblock Latent Dirichlet co-clustering.
\newblock In \emph{ICDM}.

\bibitem[{Shan and Banerjee(2008)}]{shan2008bayesian}
Shan, H.; and Banerjee, A. 2008.
\newblock Bayesian co-clustering.
\newblock In \emph{ICDE}.

\bibitem[{Tang et~al.(2015)Tang, Qu, Wang, Zhang, Yan, and Mei}]{tang2015line}
Tang, J.; Qu, M.; Wang, M.; Zhang, M.; Yan, J.; and Mei, Q. 2015.
\newblock Line: Large-scale information network embedding.
\newblock In \emph{TheWebConf}.

\bibitem[{Thakoor et~al.(2021)Thakoor, Tallec, Azar, Azabou, Dyer, Munos,
  Veli{\v{c}}kovi{\'c}, and Valko}]{thakoor2021large}
Thakoor, S.; Tallec, C.; Azar, M.~G.; Azabou, M.; Dyer, E.~L.; Munos, R.;
  Veli{\v{c}}kovi{\'c}, P.; and Valko, M. 2021.
\newblock Large-scale representation learning on graphs via bootstrapping.
\newblock \emph{arXiv preprint arXiv:2102.06514}.

\bibitem[{Van~der Maaten and Hinton(2008)}]{van2008visualizing}
Van~der Maaten, L.; and Hinton, G. 2008.
\newblock Visualizing data using t-SNE.
\newblock \emph{JMLR}.

\bibitem[{Velickovic et~al.(2019)Velickovic, Fedus, Hamilton, Li{\`o}, Bengio,
  and Hjelm}]{velickovic2019deep}
Velickovic, P.; Fedus, W.; Hamilton, W.~L.; Li{\`o}, P.; Bengio, Y.; and Hjelm,
  R.~D. 2019.
\newblock Deep Graph Infomax.
\newblock \emph{ICLR}.

\bibitem[{Vincent et~al.(2008)Vincent, Larochelle, Bengio, and
  Manzagol}]{vincent2008extracting}
Vincent, P.; Larochelle, H.; Bengio, Y.; and Manzagol, P.-A. 2008.
\newblock Extracting and composing robust features with denoising autoencoders.
\newblock In \emph{ICML}.

\bibitem[{Wang et~al.(2022)Wang, Yu, Ma, Zhang, Chen, Liu, and
  Ma}]{wang2022towards}
Wang, C.; Yu, Y.; Ma, W.; Zhang, M.; Chen, C.; Liu, Y.; and Ma, S. 2022.
\newblock Towards representation alignment and uniformity in collaborative
  filtering.
\newblock In \emph{KDD}.

\bibitem[{Wang et~al.(2023)Wang, Jing, Ding, Zhu, and
  Zhou}]{wang2023characterizing}
Wang, H.; Jing, B.; Ding, K.; Zhu, Y.; and Zhou, D. 2023.
\newblock Characterizing Long-Tail Categories on Graphs.
\newblock \emph{arXiv preprint arXiv:2305.09938}.

\bibitem[{Wang et~al.(2011)Wang, Laskey, Domeniconi, and
  Jordan}]{wang2011nonparametric}
Wang, P.; Laskey, K.; Domeniconi, C.; and Jordan, M. 2011.
\newblock Nonparametric bayesian co-clustering ensembles.
\newblock In \emph{SDM}.

\bibitem[{Wang et~al.(2021{\natexlab{a}})Wang, Hu, Wang, He, Sheng, Orgun, Cao,
  Ricci, and Yu}]{wang2021graph}
Wang, S.; Hu, L.; Wang, Y.; He, X.; Sheng, Q.~Z.; Orgun, M.~A.; Cao, L.; Ricci,
  F.; and Yu, P.~S. 2021{\natexlab{a}}.
\newblock Graph learning based recommender systems: A review.
\newblock \emph{arXiv preprint arXiv:2105.06339}.

\bibitem[{Wang et~al.(2019)Wang, He, Wang, Feng, and Chua}]{wang2019neural}
Wang, X.; He, X.; Wang, M.; Feng, F.; and Chua, T.-S. 2019.
\newblock Neural graph collaborative filtering.
\newblock In \emph{SIGIR}.

\bibitem[{Wang et~al.(2021{\natexlab{b}})Wang, Liu, Han, and
  Shi}]{wang2021self}
Wang, X.; Liu, N.; Han, H.; and Shi, C. 2021{\natexlab{b}}.
\newblock Self-supervised heterogeneous graph neural network with
  co-contrastive learning.
\newblock In \emph{KDD}.

\bibitem[{Wei and He(2022)}]{wei2022comprehensive}
Wei, T.; and He, J. 2022.
\newblock Comprehensive fair meta-learned recommender system.
\newblock In \emph{KDD}.

\bibitem[{Wei et~al.(2020)Wei, Wu, Li, Hu, Feng, He, Sun, and
  Wang}]{wei2020fast}
Wei, T.; Wu, Z.; Li, R.; Hu, Z.; Feng, F.; He, X.; Sun, Y.; and Wang, W. 2020.
\newblock Fast adaptation for cold-start collaborative filtering with
  meta-learning.
\newblock In \emph{ICDM}.

\bibitem[{Wu et~al.(2021{\natexlab{a}})Wu, Chen, Ji, and Liu}]{wu2021deep}
Wu, L.; Chen, Y.; Ji, H.; and Liu, B. 2021{\natexlab{a}}.
\newblock Deep learning on graphs for natural language processing.
\newblock In \emph{SIGIR}.

\bibitem[{Wu et~al.(2021{\natexlab{b}})Wu, Lin, Gao, Tan, and Li}]{wu2021self}
Wu, L.; Lin, H.; Gao, Z.; Tan, C.; and Li, S.~Z. 2021{\natexlab{b}}.
\newblock Self-supervised on graphs: Contrastive, generative, or predictive.
\newblock \emph{arXiv e-prints}, arXiv--2105.

\bibitem[{Xu et~al.(2019)Xu, Cheng, Zong, Ni, Song, Yu, Chen, Chen, and
  Zhang}]{xu2019deep}
Xu, D.; Cheng, W.; Zong, B.; Ni, J.; Song, D.; Yu, W.; Chen, Y.; Chen, H.; and
  Zhang, X. 2019.
\newblock Deep co-clustering.
\newblock In \emph{SDM}.

\bibitem[{Yan et~al.(2023{\natexlab{a}})Yan, Chen, Chen, Xu, Das, Yang, and
  Tong}]{yan2023trainable}
Yan, Y.; Chen, Y.; Chen, H.; Xu, M.; Das, M.; Yang, H.; and Tong, H.
  2023{\natexlab{a}}.
\newblock From Trainable Negative Depth to Edge Heterophily in Graphs.
\newblock In \emph{NeurIPS}.

\bibitem[{Yan et~al.(2023{\natexlab{b}})Yan, Jing, Liu, Wang, Li, Abdelzaher,
  and Tong}]{yan2023reconciling}
Yan, Y.; Jing, B.; Liu, L.; Wang, R.; Li, J.; Abdelzaher, T.; and Tong, H.
  2023{\natexlab{b}}.
\newblock Reconciling Competing Sampling Strategies of Network Embedding.
\newblock In \emph{NrurIPS}.

\bibitem[{Yan et~al.(2021)Yan, Liu, Ban, Jing, and Tong}]{yan2021dynamic}
Yan, Y.; Liu, L.; Ban, Y.; Jing, B.; and Tong, H. 2021.
\newblock Dynamic knowledge graph alignment.
\newblock In \emph{AAAI}.

\bibitem[{Yan, Zhang, and Tong(2021)}]{yan2021bright}
Yan, Y.; Zhang, S.; and Tong, H. 2021.
\newblock Bright: A bridging algorithm for network alignment.
\newblock In \emph{TheWebConf}.

\bibitem[{Yan et~al.(2022)Yan, Zhou, Li, Abdelzaher, and
  Tong}]{yan2022dissecting}
Yan, Y.; Zhou, Q.; Li, J.; Abdelzaher, T.; and Tong, H. 2022.
\newblock Dissecting Cross-Layer Dependency Inference on Multi-Layered
  Inter-Dependent Networks.
\newblock In \emph{CIKM}.

\bibitem[{Ying et~al.(2018)Ying, He, Chen, Eksombatchai, Hamilton, and
  Leskovec}]{ying2018graph}
Ying, R.; He, R.; Chen, K.; Eksombatchai, P.; Hamilton, W.~L.; and Leskovec, J.
  2018.
\newblock Graph convolutional neural networks for web-scale recommender
  systems.
\newblock In \emph{KDD}.

\bibitem[{You et~al.(2020)You, Chen, Sui, Chen, Wang, and Shen}]{you2020graph}
You, Y.; Chen, T.; Sui, Y.; Chen, T.; Wang, Z.; and Shen, Y. 2020.
\newblock Graph contrastive learning with augmentations.
\newblock \emph{NeurIPS}.

\bibitem[{Yu et~al.(2022)Yu, Yin, Xia, Chen, Cui, and Nguyen}]{yu2022graph}
Yu, J.; Yin, H.; Xia, X.; Chen, T.; Cui, L.; and Nguyen, Q. V.~H. 2022.
\newblock Are graph augmentations necessary? simple graph contrastive learning
  for recommendation.
\newblock In \emph{SIGIR}.

\bibitem[{Zeng et~al.(2023{\natexlab{a}})Zeng, Du, Zhang, Xia, Liu, and
  Tong}]{zeng2023hierarchical}
Zeng, Z.; Du, B.; Zhang, S.; Xia, Y.; Liu, Z.; and Tong, H. 2023{\natexlab{a}}.
\newblock Hierarchical Multi-Marginal Optimal Transport for Network Alignment.
\newblock \emph{arXiv preprint arXiv:2310.04470}.

\bibitem[{Zeng et~al.(2023{\natexlab{b}})Zeng, Zhang, Xia, and
  Tong}]{zeng2023parrot}
Zeng, Z.; Zhang, S.; Xia, Y.; and Tong, H. 2023{\natexlab{b}}.
\newblock PARROT: Position-Aware Regularized Optimal Transport for Network
  Alignment.
\newblock In \emph{TheWebConf}.

\bibitem[{Zeng et~al.(2023{\natexlab{c}})Zeng, Zhu, Xia, Zeng, and
  Tong}]{zeng2023generative}
Zeng, Z.; Zhu, R.; Xia, Y.; Zeng, H.; and Tong, H. 2023{\natexlab{c}}.
\newblock Generative graph dictionary learning.
\newblock In \emph{ICML}.

\bibitem[{Zhang and Chen(2019)}]{zhang2019inductive}
Zhang, M.; and Chen, Y. 2019.
\newblock Inductive matrix completion based on graph neural networks.
\newblock \emph{arXiv preprint arXiv:1904.12058}.

\bibitem[{Zhang et~al.(2017)Zhang, Xiong, Kong, and Zhu}]{zhang2017learning}
Zhang, Y.; Xiong, Y.; Kong, X.; and Zhu, Y. 2017.
\newblock Learning node embeddings in interaction graphs.
\newblock In \emph{CIKM}.

\bibitem[{Zhang et~al.(2023)Zhang, Liu, Zhao, Yang, Zheng, and
  Wang}]{zhang2023contrastive}
Zhang, Z.; Liu, J.; Zhao, K.; Yang, S.; Zheng, X.; and Wang, Y. 2023.
\newblock Contrastive learning for signed bipartite graphs.
\newblock In \emph{SIGIR}.

\bibitem[{Zheng et~al.(2021)Zheng, Fu, Maciejewski, and He}]{zheng2021deeper}
Zheng, L.; Fu, D.; Maciejewski, R.; and He, J. 2021.
\newblock Deeper-GXX: deepening arbitrary GNNs.
\newblock \emph{arXiv preprint arXiv:2110.13798}.

\bibitem[{Zheng et~al.(2022)Zheng, Xiong, Zhu, and
  He}]{DBLP:conf/kdd/ZhengXZH22}
Zheng, L.; Xiong, J.; Zhu, Y.; and He, J. 2022.
\newblock Contrastive Learning with Complex Heterogeneity.
\newblock In \emph{KDD}.

\bibitem[{Zheng et~al.(2023)Zheng, Zhou, Tong, Xu, Zhu, and
  He}]{zheng2023fairgen}
Zheng, L.; Zhou, D.; Tong, H.; Xu, J.; Zhu, Y.; and He, J. 2023.
\newblock FairGen: Towards Fair Graph Generation.
\newblock \emph{arXiv preprint arXiv:2303.17743}.

\bibitem[{Zheng, Zhu, and He(2023)}]{DBLP:conf/sdm/ZhengZH23}
Zheng, L.; Zhu, Y.; and He, J. 2023.
\newblock Fairness-aware Multi-view Clustering.
\newblock In \emph{SDM}.

\bibitem[{Zhou et~al.(2021)Zhou, Zhang, Yildirim, Alcorn, Tong, Davulcu, and
  He}]{zhou2021high}
Zhou, D.; Zhang, S.; Yildirim, M.~Y.; Alcorn, S.; Tong, H.; Davulcu, H.; and
  He, J. 2021.
\newblock High-order structure exploration on massive graphs: A local graph
  clustering perspective.
\newblock \emph{TKDD}.

\bibitem[{Zhou et~al.(2022)Zhou, Zheng, Fu, Han, and
  He}]{DBLP:conf/cikm/ZhouZF0H22}
Zhou, D.; Zheng, L.; Fu, D.; Han, J.; and He, J. 2022.
\newblock MentorGNN: Deriving Curriculum for Pre-Training GNNs.
\newblock In \emph{CIKM}.

\bibitem[{Zhou et~al.(2020)Zhou, Zheng, Han, and He}]{DBLP:conf/kdd/ZhouZ0H20}
Zhou, D.; Zheng, L.; Han, J.; and He, J. 2020.
\newblock A Data-Driven Graph Generative Model for Temporal Interaction
  Networks.
\newblock In \emph{KDD}.

\bibitem[{Zhu et~al.(2021)Zhu, Xu, Yu, Liu, Wu, and Wang}]{zhu2021graph}
Zhu, Y.; Xu, Y.; Yu, F.; Liu, Q.; Wu, S.; and Wang, L. 2021.
\newblock Graph contrastive learning with adaptive augmentation.
\newblock In \emph{TheWebConf}.

\end{thebibliography}
\appendix
\section{Theoretical Analysis}
We theoretically study how \sterling\ helps improve the embedding quality.
In Theorem \ref{theorem_appendix:information_bound}, we prove that $I(K;L)$ is a lower bound of $I(\mathbf{U}_\theta;\mathbf{V}_\theta)$, indicating that maximizing $I(K;L)$ could improve the connectivity of $\mathbf{U}_\theta$ and $\mathbf{V}_\theta$ in the embedding space.
We first prove a variational bound in Lemma \ref{lemma_appendix}, based on which we prove Theorem~\ref{theorem_appendix:information_bound}.

\begin{lemma_apd}[Variational Bound]\label{lemma_appendix}
For \sterling, the following inequality holds:
\begin{equation}
    \log p(k)p(l) \geq \sum_{\mathbf{u}_\theta, \mathbf{v}_\theta}p(\mathbf{u}_\theta, \mathbf{v}_\theta|k, l)\log\frac{p(\mathbf{u}_\theta, k)p(\mathbf{v}_\theta, l)}{p(\mathbf{u}_\theta, \mathbf{v}_\theta|k, l)}
\end{equation}
where $k\in\{1, \cdots, N_K\}$, $l\in\{1, \cdots, N_L\}$,  are the indices of clusters, 
$\mathbf{u}_\theta\in\mathbf{U}_\theta$, $\mathbf{v}_\theta\in\mathbf{V}_\theta$ are the node embeddings.
\end{lemma_apd}
\begin{proof}
First, we have 
\begin{equation}
p(k)=\sum_{\mathbf{u}_\theta}p(\mathbf{u}_\theta, k),\quad p(l)=\sum_{\mathbf{v}_\theta}p(\mathbf{v}_\theta, l)
\end{equation}
and therefore, we have
\begin{equation}
p(k)p(l)=\sum_{\mathbf{u}_\theta \mathbf{v}_\theta}p(\mathbf{u}_\theta, k)p(\mathbf{v}_\theta, l)
\end{equation}
As a result, we have:
\begin{equation}
\begin{split}
       &\log p(k)p(l)\\
       =& \log\sum_{\mathbf{u}_\theta, \mathbf{v}_\theta}p(k , \mathbf{u}_\theta)p(l, \mathbf{v}_\theta)\frac{p(\mathbf{u}_\theta, \mathbf{v}_\theta|k, l)}{p(\mathbf{u}_\theta, \mathbf{v}_\theta|k, l)}\\
       \geq& \sum_{\mathbf{u}_\theta, \mathbf{v}_\theta}p(\mathbf{u}_\theta, \mathbf{v}_\theta|k, l)\log\frac{p(\mathbf{u}_\theta, k)p(\mathbf{v}_\theta, l)}{p(\mathbf{u}_\theta, \mathbf{v}_\theta|k, l)} 
\end{split}
\end{equation}
where the inequality holds according to Jensen's inequality.
\end{proof}

\begin{theorem}[Information Bound]\label{theorem_appendix:information_bound}
The mutual information $I(\mathbf{U}_\theta;\mathbf{V}_\theta)$ of embeddings $\mathbf{U}_\theta$ and $\mathbf{V}_\theta$ is lower-bounded by the mutual information of co-clusters $I(K;L)$:
\begin{equation}\label{eq:theorem_bound2}
    I(K;L)\leq I(\mathbf{U}_\theta;\mathbf{V}_\theta)
\end{equation}
\end{theorem}
\begin{proof}
According to Lemma \ref{lemma_appendix}, we have:
\begin{equation}\label{eq:theorem_1}
\begin{split}
    & I(K; L)\\ 
    =& \sum_{k, l}p(k, l)\frac{\log p(k, l)}{\log p(k)p(l)}\\
    \leq& \sum_{k, l}p(k, l)\sum_{\mathbf{u}_\theta, \mathbf{v}_\theta}p(\mathbf{u}_\theta, \mathbf{v}_\theta|k, l)\log \frac{p(k, l)p(\mathbf{u}_\theta, \mathbf{v}_\theta|k, l)}{p(\mathbf{u}_\theta, k)p(\mathbf{v}_\theta, l)}\\
    =& \sum_{\mathbf{u}_\theta, \mathbf{v}_\theta, k, l}p(\mathbf{u}_\theta, \mathbf{v}_\theta, k, l)\log \frac{p(\mathbf{u}_\theta, \mathbf{v}_\theta, k, l)}{p(\mathbf{u}_\theta, k)p(\mathbf{v}_\theta, l)}\\
    =& \sum_{\mathbf{u}_\theta, \mathbf{v}_\theta, k, l}p(\mathbf{u}_\theta, \mathbf{v}_\theta, k, l)\log \frac{p(k, l|\mathbf{u}_\theta, \mathbf{v}_\theta)p(\mathbf{u}_\theta, \mathbf{v}_\theta)}{p(\mathbf{u}_\theta)p(k|\mathbf{u}_\theta)p(\mathbf{v}_\theta)p(l|\mathbf{v}_\theta)}\\ =& R
\end{split}
\end{equation}
Since $\mathcal{C}_\theta^U\circ\mathcal{P}_\theta^U$ and $\mathcal{C}_\theta^V\circ\mathcal{P}_\theta^V$ have separate sets of parameters and inputs, therefore, it is natural to have $p(k, l|\mathbf{u}_\theta, \mathbf{v}_\theta)=p(k|\mathbf{u}_\theta)p(l|\mathbf{v}_\theta)$.
As a result, we have:
\begin{equation}\label{eq:theorem_2}
\begin{split}
    R &= \sum_{\mathbf{u}_\theta, \mathbf{v}_\theta, k, l}p(\mathbf{u}_\theta, \mathbf{v}_\theta, k, l)\log \frac{p(\mathbf{u}_\theta, \mathbf{v}_\theta)}{p(\mathbf{u}_\theta)p(\mathbf{v}_\theta)}\\
    &= \sum_{\mathbf{u}_\theta, \mathbf{v}_\theta}p(\mathbf{u}_\theta, \mathbf{v}_\theta)\log \frac{p(\mathbf{u}_\theta, \mathbf{v}_\theta)}{p(\mathbf{u}_\theta)p(\mathbf{v}_\theta)} = I(\mathbf{U}_\theta;\mathbf{V}_\theta)
\end{split}
\end{equation}
Combine \eqref{eq:theorem_1} and \eqref{eq:theorem_2}, and we have \eqref{eq:theorem_bound2}. 
\end{proof}



\section{Experimental Setup}
\subsection{Datasets}
Descriptions of datasets are presented in Table \ref{tab_apd:data}.
\textbf{ML-100K}\footnote{\url{https://grouplens.org/datasets/movielens/100k/}}
is collected via the MovieLens\footnote{\url{https://movielens.org/}} 
website, which contains 100K movie ratings from 943 users on 1682 movies. 
Each user has rated at least 20 movies, and the relation between users and items are converted to binary \cite{cao2021bipartite}.
This dataset is used for the recommendation task.
\textbf{Wikipedia}\footnote{\url{https://github.com/clhchtcjj/BiNE/tree/master/data/wiki}} contains the edit relationship between authors and pages, which is used for link prediction
, and it has two different splits (50\%/40\%) for training~\cite{cao2021bipartite}.
\textbf{IMDB}\footnote{\url{https://github.com/dongkuanx27/Deep-Co-Clustering/blob/master/DeepCC/Data/IMDb_movies_keywords.mat}}, \textbf{Cornell}\footnote{\url{https://github.com/dongkuanx27/Deep-Co-Clustering/blob/master/DeepCC/Data/WebKB_cornell.mat}}, and  \textbf{Citeseer}\footnote{\url{https://linqs.org/datasets/\#citeseer-doc-classification}} 
are document-keyword bipartite graphs, which are used for the co-clustering task~\cite{xu2019deep}.


\begin{table*}[t]
    \centering
    \footnotesize
    \scalebox{1.0}{\begin{tabular}{cccccccc}
        \hline
        Dataset & Task & Evaluation & $|U|$ & $|V|$ & $|E|$ & \# Class\\
        \hline
        ML-100K & Recommendation & F1, NDCG, MAP, MRR & 943 & 1,682 & 100,000 & -\\
        \hline
        Wikipedia & Link Prediction & AUC-ROC & 15,000 & 3,214 & 64,095 & -\\
        \hline
        IMDB & Co-Clustering & NMI, ACC & 617 & 1878 & 20,156 & 17\\
        Cornell & Co-Clustering & NMI, ACC & 195 & 1,703 & 18,496 & 5\\
        Citeseer & Co-Clustreing & NMI, ACC & 3,312 & 3,703 & 105,165 & 6\\
        \hline
    \end{tabular}}
    \caption{Descriptions of the datasets.}
    \label{tab_apd:data}
\end{table*}

\begin{table*}[t]
    \centering
    \footnotesize
    \scalebox{1.0}{\begin{tabular}{cccccccccccccc}
        \hline
        Dataset &$N_{knn}$ & $N_K=N_L$ & $\alpha$ & $d$ & $L$ & Skip Conn. & $\mathcal{P}$ & $n$ & lr & Epoch \\
        \hline
        ML-100K & 10 & 10 & 0 & 2,048 & 1 & False & Identity & 2 & $5\times 10^{-4}$ & 10\\
        \hline
        Wikipedia & 10 & 10 & -0.8 & 512 & 2 & True & MLP & 3 & $1\times 10^{-4}$ & 20\\
        \hline
        IMDB & 10 & 100 & -1 & 2048 & 1 & True & MLP & 1 & $5\times 10^{-4}$ & 50\\
        Cornell & 10 & 100 & -1 & 2048 & 1 & True & MLP & 1 & $5\times 10^{-4}$ & 10\\
        Citeseer & 10 & 100 & -1 & 2048 & 1 & True & MLP & 1 & $5\times 10^{-4}$ & 10\\
        \hline
    \end{tabular}}
    \caption{Hyper-parameters.}
    \label{tab:detailed_settings}
\end{table*}

\subsection{Neural Network Architecture}
\paragraph{Encoder.}
The encoder $\mathcal{E}$ 
(subscripts $\theta$ and $\phi$ are dropped for clarity) 
is a simple $L$-layer message passing model.
The updating functions for $\mathbf{u}$ are presented in Equation \eqref{eq:encoder_1}\eqref{eq:encoder_2}.
$\mathbf{v}$ is obtained via similar updating functions as $\mathbf{u}$ but with a different set of parameters.
\begin{equation}\label{eq:encoder_1}
    \mathbf{u}^{(l+1)} =\delta( \mathbf{W}_1^{(l+1)}\cdot\text{Mean}(\{\mathbf{v}^{(l)}:e_{uv}\in E\}))
\end{equation}
where $\delta$ is an activation function, $\mathbf{W}_1^{(l+1)}$ is the weight of the $l+1$-th layer, $\mathbf{v}^{(0)}$ is a randomly initialized embedding vector.
We set $\delta$ as ReLU throughout the experiments.
There is an optional skip connection to prevent potential over-smoothing:
\begin{equation}\label{eq:encoder_2}
    \mathbf{u}^{(l+1)} =\delta(\mathbf{W}_{2}^{(l+1)}[\mathbf{u}^{(l+1)}||\mathbf{u}^{(l)}])
\end{equation}
where $\mathbf{W}^{(l+1)}_{2}$ is the weight and $||$ is the concatenation operation.
We set $\delta$ as Tanh throughout the experiments.

\paragraph{Projector.}
The projectors $\mathcal{P}^U$ and $\mathcal{P}^V$ are either implemented by two-layer Multi-Layer Perceptron (MLPs) with Tanh activation or an identity mapping function.

\paragraph{Clustering Network.}
The clustering networks $\mathcal{C}^U$ and $\mathcal{C}^V$ are implemented by two-layer MLPs.
The activation functions of the fist and second layers are Tanh and Softmax.

\subsection{Implementation Details} 
Following \cite{grill2020bootstrap}, we update the parameter $\tau$ of Exponential Moving Average (EMA) by:
\begin{equation}
    \tau \leftarrow 1 - (1-\tau_{init})\cdot(\cos(\pi k/K)+1)/2
\end{equation}
where $k$ and $K$ are current and maximum training epochs, and we set the initial value of $\tau$ as $\tau_{init}=0.99$.
We use absolute activation function for $\mathbf{A}_{emb}$ in Equation \eqref{eq:a_emb}. 
The other settings are presented in Table \ref{tab:detailed_settings}, including the number of k-NN ($N_{knn}$), the number of co-clusters $N_K=N_L$, filter threshold $\alpha$, the number of the encoder layers $L$, whether use skip connections in encoder, whether use identity mapping or MLP for $\mathcal{P}$, the order of meta-path $n$ for $\mathbf{A}_{meta}$, learning rate and the number of epochs for training.
The coefficients
$\lambda_{uv}$, $\lambda_{u}$ and $\lambda_{v}$ are searched within [0.001, 0.01, 0.1, 1, 2, 5].
The code is implemented by PyTorch.
All the experiments are conudcted on a single Nvidia Tesla V-100 32G GPU.

\end{document}